\documentclass[12pt,twoside]{report}

\newcommand{\reporttitle}{A Deep Autoregressive Model for Dynamic Combinatorial Complexes}

\newcommand{\degreetype}{Artificial Intelligence and Machine Learning}

\usepackage[a4paper,hmargin=2.8cm,vmargin=2.0cm,includeheadfoot]{geometry}
\usepackage{textpos}
\usepackage{tabularx,longtable,multirow,subfigure,caption}%hangcaption
\usepackage{fancyhdr} % page layout
\usepackage{url} % URLs
\usepackage[english]{babel}
\usepackage{amsmath}
\usepackage{graphicx}
\usepackage{dsfont}
\usepackage{epstopdf} % automatically replace .eps with .pdf in graphics
\usepackage{backref} % needed for citations
\usepackage{array}
\usepackage{latexsym}
\usepackage{amsthm}
\usepackage{algorithm}
\usepackage{algorithmic}

\usepackage[pdftex,pagebackref,hypertexnames=false,colorlinks]{hyperref} % provide links in pdf
\usepackage{booktabs}

\hypersetup{pdftitle={},
  pdfsubject={}, 
  pdfauthor={},
  pdfkeywords={}, 
  pdfstartview=FitH,
  pdfpagemode={UseOutlines},% None, FullScreen, UseOutlines
  bookmarksnumbered=true, bookmarksopen=true, colorlinks,
    citecolor=black,%
    filecolor=black,%
    linkcolor=black,%
    urlcolor=black}

\usepackage[all]{hypcap}

\usepackage{color}
\usepackage[tight,ugly]{units}
\usepackage{float}
\usepackage{tcolorbox}
\usepackage[colorinlistoftodos]{todonotes}

\newtheorem{proposition}{Proposition}
\newtheorem{remark}{Remark}
\newtheorem{definition}{Definition}
\newtheorem{example}{Example}
\usepackage{amssymb}
\usepackage{multicol} % For side-by-side text

\def\@makechapterhead#1{%
  \vspace*{10\p@}%
  {\parindent \z@ \raggedright \sffamily
    \interlinepenalty\@M
    \Huge\bfseries \thechapter \space\space #1\par\nobreak
    \vskip 30\p@
  }}

\def\@makeschapterhead#1{%
  \vspace*{10\p@}%
  {\parindent \z@ \raggedright
    \sffamily
    \interlinepenalty\@M
    \Huge \bfseries  #1\par\nobreak
    \vskip 30\p@
  }}

\allowdisplaybreaks
\newcommand{\R}[0]{\mathds{R}} % real numbers
\newcommand{\B}[0]{\mathds{B}} % binary numbers
\newcommand{\Prob}{\mathds{P}}

\def\eqref#1{equation~\ref{#1}}
\def\1{\bm{1}}

\DeclareMathOperator{\CC}{CC}
\newcommand{\CCX}{\mathcal{X}}
\newcommand{\CCY}{\mathcal{Y}}
\newcommand{\CCZ}{\mathcal{Z}}

\newcommand{\N}{\mathbb{N}} % natural numbers
\newcommand{\Znon}{\mathbb{Z}_{\ge 0}} % nonnegative integers 

\DeclareMathOperator{\rk}{rk}

\DeclareMathOperator{\Int}{int}

\DeclareMathOperator{\CCN}{CCNN}

\date{September 2024}

\begin{document}

% load title page
% Last modification: 2015-08-17 (Marc Deisenroth)
\begin{titlepage}

\newcommand{\HRule}{\rule{\linewidth}{0.5mm}} % Defines a new command for the horizontal lines, change thickness here

%----------------------------------------------------------------------------------------
%	LOGO SECTION
%----------------------------------------------------------------------------------------

\includegraphics[width = 4cm]{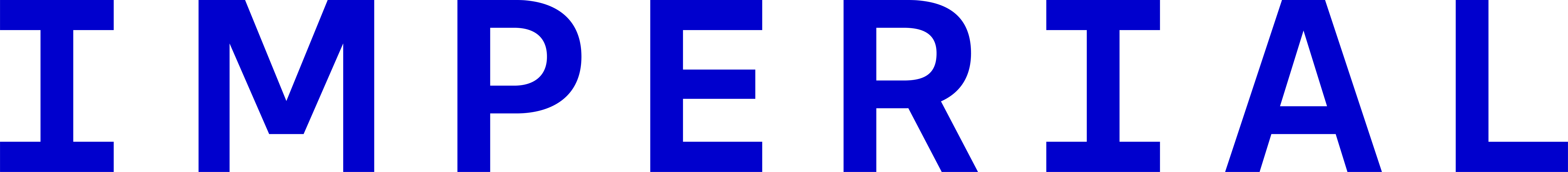}\\[0.5cm] 

\center % Center remainder of the page

%----------------------------------------------------------------------------------------
%	HEADING SECTIONS
%----------------------------------------------------------------------------------------

\textsc{\Large Imperial College London}\\[0.5cm] 
\textsc{\large Department of Computing}\\[0.5cm] 

%----------------------------------------------------------------------------------------
%	TITLE SECTION
%----------------------------------------------------------------------------------------

\HRule \\[0.4cm]
{ \huge \bfseries \reporttitle}\\ % Title of your document
\HRule \\[1.5cm]
 
%----------------------------------------------------------------------------------------
%	AUTHOR SECTION
%----------------------------------------------------------------------------------------

\begin{minipage}{0.4\textwidth}
\begin{flushleft} \large
\emph{Author:}\\
Ata Tuna % Your name
\end{flushleft}
\end{minipage}
~
\begin{minipage}{0.4\textwidth}
\begin{flushright} \large
\emph{Supervisor:} \\
Dr. Tolga Birdal % Supervisor's Name
\end{flushright}
\end{minipage}\\[4cm]

%----------------------------------------------------------------------------------------
%	FOOTER & DATE SECTION
%----------------------------------------------------------------------------------------
\vfill % Fill the rest of the page with whitespace
Submitted in partial fulfillment of the requirements for the MRes degree in
\degreetype~of Imperial College London\\[0.5cm]

\makeatletter
\@date 
\makeatother

\end{titlepage}

% page numbering etc.
\pagenumbering{roman}
\clearpage{\pagestyle{empty}\cleardoublepage}
\setcounter{page}{1}
\pagestyle{fancy}

%%%%%%%%%%%%%%%%%%%%%%%%%%%%%%%%%%%%
\begin{abstract}
We introduce DAMCC (Deep Autoregressive Model for Dynamic Combinatorial Complexes), the first deep learning model designed to generate dynamic combinatorial complexes (CCs). Unlike traditional graph-based models, CCs capture higher-order interactions, making them ideal for representing social networks, biological systems, and evolving infrastructures. While existing models primarily focus on static graphs, DAMCC addresses the challenge of modeling temporal dynamics and higher-order structures in dynamic networks.

DAMCC employs an autoregressive framework to predict the evolution of CCs over time. Through comprehensive experiments on real-world and synthetic datasets, we demonstrate its ability to capture both temporal and higher-order dependencies. As the first model of its kind, DAMCC lays the foundation for future advancements in dynamic combinatorial complex modeling, with opportunities for improved scalability and efficiency on larger networks.

\end{abstract}

\cleardoublepage
%%%%%%%%%%%%%%%%%%%%%%%%%%%%%%%%%%%%
\section*{Acknowledgments}
I thank my supervisor Tolga Birdal for introducing me to Topological Deep Learning and igniting my interest in the subject. Most importantly, I would like to thank my parents for their unwavering support throughout my prolonged academic journey.

\clearpage{\pagestyle{empty}\cleardoublepage}

%%%%%%%%%%%%%%%%%%%%%%%%%%%%%%%%%%%%
%--- table of contents
\fancyhead[RE,LO]{\sffamily {Table of Contents}}
\tableofcontents

\clearpage{\pagestyle{empty}\cleardoublepage}
\pagenumbering{arabic}
\setcounter{page}{1}
\fancyhead[LE,RO]{\slshape \rightmark}
\fancyhead[LO,RE]{\slshape \leftmark}

\setlength{\parskip}{1em}  % Adds vertical space between paragraphs

%%%%%%%%%%%%%%%%%%%%%%%%%%%%%%%%%%%%

\chapter{Introduction}

Networks are powerful tools for modeling various real-life phenomena due to their ability to capture relationships and interactions between entities in a structured manner. In real life, many complex systems, from social interactions and transportation systems to biological processes and technological infrastructures, can be represented as networks. By modeling these systems as networks, we can gain insights into their behaviour, identify key components, and even predict future developments. Networks allow for the analysis of dependencies, patterns, and flows of information or resources, making them very useful in numerous domains such as epidemiology \cite{panagopoulos2021transfergraphneuralnetworks}, chemistry \cite{guo2023graphbasedmolecularrepresentationlearning}, biology \cite{li2022graphrepresentationlearningbiomedicine}, communication \cite{suárezvarela2022graphneuralnetworkscommunication}, and social science \cite{Goldenberg2021SocialNA}. Thus, network generation is a useful tool for both theoretical advancements and practical applications, as it enables us to explore, simulate, and predict the behaviour of systems \cite{taggen,gupta_tigger_2022,clarkson_damnets_2023, inbook}.

In the past decade, generative deep learning, particularly in natural language processing and image generation, yielded groundbreaking results via Deep Belief Networks (DBNs) \cite{Hinton2006AFL}, Generative Adversarial Networks (GANs) \cite{GANs}, Variational Autoencoders (VAEs) \cite{kingma2022autoencoding}, Diffusion Models \cite{ho2020denoisingdiffusionprobabilisticmodels},  Auto-Regressive Models \cite{oord2016pixel}, and Transformer Models \cite{vaswani2023attention} to name a few. These models were adapted to simple topological structures called graphs to model molecules \cite{decao2022molganimplicitgenerativemodel, samanta2019nevaedeepgenerativemodel}, social networks \cite{bojchevski2018netgangeneratinggraphsrandom}, traffic \cite{cui2019trafficgraphconvolutionalrecurrent}, and many other systems. 

However, network phenomena as they exist in nature and society are too complex to be represented by binary relationships, requiring more flexible domains for analysis. For instance, in a party setting, interactions may involve more than one person; molecules may experience multiple forces and more complex interactions; social media users interact via chats, group chats, feeds, hashtags, pages, and so on.

Combinatorial complexes (CCs) generalise some topological structures (discussed in the preliminaries). CCs can better capture and represent higher-order relationships that exist in real-life phenomena modeling multi-faceted interactions and hierarchical structures that are inherent in many natural and artificial systems. We explain CCs in detail in Chapter~\ref{ccs}. Networks modeled as CCs have the potential to be used for simulations, predictions, and analyses that mirror the complexities of the real world, ultimately leading to more accurate models and better decision making tools. To our knowledge, the only attempt at a combinatorial complex generative model in the literature is CCSD \cite{carrel_combinatorial_2024}.

Many of the networks described above, in fact, have a temporal element regarding their evolution. Social networks, for example, change over time as new connections are formed and old ones dissolve. Similarly, transportation networks evolve with the construction of new routes or closures due to maintenance. In biological systems, gene regulatory networks or neural connections can dynamically change in response to external stimuli or developmental processes. Despite the inherent temporal dynamics in these systems, much of the current research on network generation focuses on static models, ignoring the crucial aspect of how networks evolve over time. This gap in the literature calls for more attention to temporal networks, where both the structure and the timing of interactions are essential for accurate modeling. Understanding how networks change over time can provide deeper insights into processes such as the spread of diseases, information diffusion, or network resilience in the face of disruptions. Temporal network generation models are therefore a promising avenue for capturing the full complexity of real-world phenomena, allowing us to better understand and predict the dynamic behaviour of evolving systems. It is surprising how little attention temporal (or dynamic) networks have received in machine learning literature. Our model DAMCC (A Deep Autoregressive Model for Dynamic Combinatorial Complexes), influenced by DAMNETS \cite{clarkson_damnets_2023}, is our attempt to address this gap by extending network generation to dynamic combinatorial complexes. While many existing models focus on static graphs, and rarely dynamic graphs, DAMCC aims to not only accommodate higher-order interactions that are inherently present in real-world systems, but it also attempts to capture the network's evolution over time. By leveraging a deep autoregressive framework, DAMCC can model both the creation and dissolution of higher-order relationships, accounting for temporal patterns and complex dependencies that static graph based models overlook.

\section{Objective}
\label{objective}

The objective of this report is to present a model capable of generating a prediction of the next combinatorial complex given a time series of combinatorial complexes. 

More formally, we have a sequence or multiple sequences of length $T$ of CCs:

\[(CC_t)_{t=0}^T = (CC_1, CC_2, \dots, CC_T). \]

We assume that $(CC_t)_{t=0}^T$ is Markovian, i.e. all the information that can be inferred regarding a CC is contained in the previous CC in the sequence. So we assume that each $CC_{t}$ is a random variable and 

\[CC_{t} \sim p(\cdot \mid  CC_{t-1}) = p(\cdot \mid  CC_{t-1}, CC_{t-2}, \dots, CC_{0}).\]

 The goal of our model is to learn an approximation  $\hat{p}(\cdot \mid  (CC_{t-1}))$ from which we can sample a prediction $\hat{CC}_{t}$.

\section{Outline of This Report}

\textbf{Chapter 1: Introduction} introduces the motivation for network generation, particularly focusing on combinatorial complexes and their relevance to real-world phenomena. It also outlines the objectives and the contributions of this work.

\textbf{Chapter 2: Related Work} surveys the existing literature on network generation models. This includes models like graph neural networks (GNNs), simplicial complexes, hypergraphs, and other relevant higher-order models, along with a focus on temporal networks and their generation. We highlight key differences and the gaps in temporal network modeling that the proposed model addresses.

\textbf{Chapter 3: Preliminaries} provides the mathematical and theoretical foundations required for understanding combinatorial complexes and topological deep learning. This includes key definitions and concepts such as combinatorial complexes, cellular complexes, simplicial complexes, hypergraphs, and graphs.

\textbf{Chapter 4: Methodology \& Contribution} presents the primary contributions of this work, including the design of a novel autoregressive model for temporal combinatorial complex generation (DAMCC). We also introduce the loss functions used for optimization and describe the architecture of the model, which leverages higher-order attention mechanisms.

\textbf{Chapter 5: Experiments} describes the experimental setup, including the datasets used, evaluation metrics, and optimization strategies. We compare the proposed DAMCC model with existing state-of-the-art models on a range of synthetic and real-world datasets, followed by an in-depth analysis of the results.

\textbf{Chapter 6: Conclusion} summarizes the main findings of the report, discusses the implications of the research, and provides suggestions for future work.

Appendices provide additional details on mathematical formulations, algorithms, and code implementations used throughout the study.

%%%%%%%%%%%%%%%%%%%%%%%%%%%%%%%%%%%%
\chapter{Related Work}

\section{Static Networks Across Various Domains} \label{current}

\subsection{GNNs} 

As graphs are the oldest domain among the ones considered in this review, they also have the highest number of papers published utilising them. Graph based protein language models have had success with \cite{Jha2022, Jha2023}, where the ``input to the language model was the protein sequence, and the output was the feature vector for each amino acid of the underlying sequence" \cite{Jha2022}. \cite{Zhangtraffic} employed a Kernel-Weighted Graph Convolutional Network (KW-GCN) approach and applied it to commute and traffic modeling, where the model learns convolutional kernels and their linear combination weights on nodes in the graph, which significantly outperformed previous state of the art models. \cite{Amar2014TitleC} demonstrated that module maps were very suitable for the analysis of heterogeneous omic data. Edge-weighted graphs have been successfully utilised to model infectious diseases \cite{Manriquez2021-cu, Manriquez2021-ur}. \cite{Yan_Xiong_Lin_2018} introduced Spatial-Temporal Graph Convolutional Networks (ST-GCN) for human action recognition and outperformed the main models available at the time. Regarding graph network generation, recent works are \cite{li2018learningdeepgenerativemodels, you2018graphrnngeneratingrealisticgraphs, liao2020efficientgraphgenerationgraph}, and most importantly, BiGG \cite{dai_scalable_2020}, whose approach is very similar to ours.

\subsection{Simplicial Complexes and Hypergraphs}

These domains found significant use in signal processing. \cite{Barbarossa2018LEARNINGFS} expands the tools developed on graphs to the analysis of signals defined on simplicial complexes. The same authors published \cite{Barbarossa2019TopologicalSP} with further developed techniques over SCs and applied them to traffic analysis over wireless networks. Further advancements in signal processing came from Hodge theory where \cite{SCHAUB2021108149} utilised Hodge Laplacian matrices generalised properties of the Laplacian matrix using simplicial complexes.

\cite{DBLP:journals/corr/abs-2112-10570} proposed dynamic hypergraph
convolutional networks (DHGCN) and came up with a human skeleton based action recognition model. \cite{dynamicjiang, feng2019hypergraph} proposed a framework with hypergraph neural networks  (DHGNN) made of concatenated modules: dynamic hypergraph construction (DHG) and hypergraph convolution (HGC). \cite{Arya} made use of hypergraphs representing social networks to model relationships. ``UniGNN, a unified framework for interpreting
the message passing process in graph and hypergraph neural networks'' was further proposed by \cite{Huang2021UniGNNAU}. \cite{yang2023convolutional} is a very recent convolutional model defined on simplices concerned with link prediction and trajectory classification. \cite{pmlr-v198-yang22a} explored many-body interaction models via SCs.

\subsection{Cellular and Combinatorial Complexes}

Cell Attention Networks (CANs) were proposed by \cite{giusti2022cell}, which is ``a neural architecture operating on data defined over the vertices of a graph, representing the graph as the 1-skeleton of a cell complex introduced to capture higher-order interactions.''
CCs were introduced as recently as 2022 with \cite{hajij2023topological}, and there have been very few publications regarding them. \cite{hajij2023topological} provided  Python modules for topological deep learning: TopoNetX \cite{TopoNetX}, TopoEmbedX \cite{TopoEmbedX} and TopoModelX \cite{TopoModelX}. \cite{Carrel_CCSD_-_Combinatorial_2023} is the most recent paper introducing a score-based generative model on CCs. At the time of the submission of this paper, no other generative model regarding CCs exists within the literature. 

\section{Temporal Network Generation}

Despite their potential use, temporal networks are surprisingly neglected. Within the literature, the only temporal networks we have found were defined on graphs. We list some earlier models first and then explain in detail the most relevant ones to us. The earliest attempts at network time series generation were not capable of arbitrary network generation but rather focused on hard coding to capture a certain property of a network. One of the first models to attempt to model arbitrary networks was TagGen \cite{taggen}, a random walk based model with a self attention mechanism. DYMOND \cite{zeno_dymond_2021} is another notable model that attempts to model the arrival of subgraphs called ``motifs''. DAMNETS paper compared their performance to these models as well as AGE. DYMOND, TagGen, TIGGER\cite{gupta_tigger_2022}, and D2G2\cite{zhang2021disentangleddynamicgraphdeep} take an entire network series and generate an entire network series similar to that series. Since their task is different, the DAMNETS paper let DYMOND and TagGen have access to the test data during the experiments. Still, DAMNETS outperformed them. \cite{souid_temporal_2024} is a paper which does not introduce a new model, but introduces a novel evaluation metric appropriated to graph time series. \cite{souid_temporal_2024} therefore separates DAMNETS-AGE, DYMOND-TIGGER-D2G2 evaluations. Since their task is different, we won't include DYMOND, D2G2, TIGGER, and TagGen in our experiments in Chapter~\ref{sec:experiments}, and compare only DAMNETS, AGE, and DAMCC. Regarding our task, there exist no models in the literature which can generate updates of CC time series. The closest proxies are only AGE and DAMNETS, and we compare our model's performance against these two only with the addition of a ``random model''. CCSD \cite{carrel_combinatorial_2024} is also worth mentioning as it is the only CC generative model. CCSD took a diffusion model based approach by representing CC's as a number of very large incidence matrices, which included every possible row, approximated the score function of this diffusion process, and cut off values below a threshold to sample a binary incidence matrix. This approach is computationally very expensive, especially for larger graphs that we are interested in.  Therefore, we base our sampling method on DAMNETS, which is based on BiGG. We first explain AGE, as it is simpler than DAMNETS. Then we explain BiGG before DAMNETS as DAMNETS borrows the decoder mostly from BiGG.

\textbf{AGE} \cite{inbook} follows an encoder-decoder structure. AGE generates an adjacency matrix row by row, and the rows element by element. It first concatenates node features with the adjacency matrix. It then encodes every node to a high level embedding. The decoder takes this encoded matrix along with the target graph adjacency matrix. Then it decodes every row element by element, i.e. one node at a time, by several stacked attention modules that alternate self attention and source-target attention layers.

\textbf{BiGG} \cite{dai_scalable_2020} works on adjacency matrices, sampling row by row. Every row sampled depends on the previous rows sampled in an autoregressive manner. With a second autoregressive component, it samples $1$'s across the row via a binary tree $\mathcal{T}$. $\mathcal{T}$ operates as follows: each node of this tree is responsible for a range of graph nodes \([n_l, n_r]\). Starting the root tree node $[1,n]$, at every tree node the tree makes a decision as to whether to eliminate the left half \([n_l, \left\lfloor\frac{n_l+n_r}{2}\right\rfloor]\) of the graph nodes. After that, the tree decides for the right half \([\left\lceil\frac{n_l+n_r}{2}\right\rceil, n_r] \). If the tree decides not to eliminate, we descend to the tree node of the half we decided to not eliminate, and we iterate the process. BiGG tracks down this traverse via state variables produced by Transformers\cite{vaswani2023attention}, and RNNs\cite{Rumelhart1986LearningRB}. We explain this process in detail in Algorithm~\ref{algo2} and Example~\ref{tree}. Therefore, for any row in the adjacency matrix, we sample in $O(\mathcal{T})$ decision steps. If $N$ is the number of graph nodes, since the maximum depth of the tree is $\log(N)$, we have the upper bound $|\mathcal{T}| \leq N \log(N)$. In practice, if adjacency matrices are sparse, this algorithm samples efficiently by eliminating a portion of the graph nodes at every decision. This is the approach in DAMNETS and DAMCC.

\textbf{DAMNETS} \cite{clarkson_damnets_2023} first encodes every graph node via GAT\cite{velickovic_graph_2018} and then decodes these embeddings via a very similar tree structure to BiGG with the same traverse but also improves the process the sampling delta matrices $\Delta$ defined by:

\[
\Delta^{t} = A^t-A^{t-1},
\]

where $A^t$ is the adjacency matrix (Definition~\ref{def:adj}) of timestep $t$, therefore taking advantage of the sparsity of $\Delta$. At every tree leaf, DAMNETS also augments BiGG by making an extra decision as to whether to sample the leaf or not, which corresponds to a change from the previous graph.

%%%%%%%%%%%%%%%%%%%%%%%%%%%%%%%%%%%%
\chapter{Preliminaries}

\section{Topological Deep Learning}

\subsection{Background and Emergence: Early Developments}

Strictly speaking, Topology is a discipline in pure mathematics and a sub-discipline in algebra which studies geometric objects that preserve certain properties under continuous deformations. For a more mathematical study, \cite{postol2023algebraic} is an up-to-date textbook. We will not delve deep into pure mathematics in this review except for some useful mathematical constructs borrowed from Topology.

Topology's first use for ML purposes first emerged in the form of Topological Data Analysis (TDA) in the highly influential papers \cite{Edelsbrunner2002, Zomorodian2005} which are seminal in the development of persistent homology, an indispensable method early on the field for extraction of features from more complex datasets \cite{pun2018persistenthomologybased}. Their work was later developed and popularized in a landmark article in 2009 \cite{Carlsson2009}. ``TDA is mainly motivated by the idea that topology and geometry provide a powerful approach to infer robust qualitative, and sometimes quantitative, information about the structure of data" \cite{Chazal21}.
Integration of TDA with Machine Learning has been a gradual progress and initially constrained to data analysis and feature extraction rather than integration within learning architectures \cite{Nicolau2011TopologyBD, lima2023image}.

\section{Foundations of Topological Deep Learning}

Since TDL is a newly emerging field, terminology and notation for the newly invented mathematical concepts are yet to be universal. To remedy this, \cite{hajij2023topological} aimed to provide a unified framework for the future papers in the field. This chapter will follow the definitions outlined in that document. \cite{papillon2023architectures} follows as a more succinct survey of the architectures developed so far. 

Methods before TDL relied on data defined in traditional data structures. TDL allows this data to be defined on different \emph{domains}, such as a graph or a set, or novel objects such as \emph{Simplicial Complexes} (SCs), \emph{Cellular Complexes (CWs)}, \emph{Hypergraphs} and most recently, \emph{Combinatorial Complexes} (CCs), which generalise all the aforementioned and thus provide ultimate flexibility in data representations.

In the following sections, we introduce these aforementioned objects closely following the definitions in \cite{hajij2023topological} mostly verbatim with occasional fixes and adaptations, and provide remarks as to how these definitions are to be first interpreted and build up to \emph{Combinatorial Complexes CCs}. We include these definitions in this work so that a reader who is new to the subject will not have to constantly go back and forward between this work and \cite{hajij2023topological}, and that this report is mostly self-contained. We follow by introducing Topological Neural Networks defined on these domains. Let us begin.

\subsection{Definitions Building Up To Combinatorial Complexes}

\begin{tcolorbox}
[width=\linewidth, sharp corners=all, colback=white!95!black]
\begin{definition}[Neighbourhood function]
\label{NS}
Let $S$ be a nonempty set.
A \textbf{neighbourhood function} on $S$
is a function $\mathcal{N}\colon S\to\mathcal{P}(\mathcal{P}(S))$
that assigns to each point $x$ in $S$
a nonempty collection $\mathcal{N}(x)$ of subsets of $S$, where $\mathcal{P}$ denotes the powerset.
The elements of $\mathcal{N}(x)$ are called \textbf{neighbourhoods} of $x$ with respect to $\mathcal{N}$.
\end{definition}
\end{tcolorbox}

    Notice that the definition above says not much about the structure of the underlying set. According to the above definition, it is just a function which maps an element from a set to a collection of subsets of that set. Intuitively, a neighbourhood is simply a subset of $\mathcal{S}$. We need to bring some constraints over the above definition to make it a useful tool. The following definition is a start at that.

\begin{tcolorbox}
 [width=\linewidth, sharp corners=all, colback=white!95!black]
\begin{definition}[Neighbourhood topology]
\label{NT}
Let $\mathcal{N}$ be a neighbourhood function on a set $S$.
$\mathcal{N}$ is called a \textbf{neighbourhood topology} on $S$ if it satisfies the following axioms:
\begin{enumerate}
\item If $N$ is a neighbourhood of $x$, then $x\in N$.
%For every neighbourhood $N \in \mathcal{N}(x)$ of a point $x \in S$, it holds that $x\in N$.
% \item If $N \in \mathcal{N}(x)$ is a neighbourhood of $x$, then $x\in N$.
\item If $N$ is a subset of $S$ containing a neighbourhood of $x$,
then $N$ is a neighbourhood of $x$.
%Every superset of a neighbourhood of a point $x \in S$ is a neighbourhood of $x$.
\item The intersection of two neighbourhoods of a point $x$ in $S$ is a neighbourhood of $x$.
\item Any neighbourhood $N$ of a point $x$ in $S$ contains a neighbourhood $M$ of $x$ such that $N$ is a neighbourhood of each point of $M$.
\end{enumerate}
\end{definition}
\end{tcolorbox}

    Intuitively, a neighbourhood topology defines some sort of proximity relationship in a set. This is going to be very useful for the construction of higher-order deep learning models as we would like to define certain mathematical structures that represent data in real life, and define a neighbourhood topology on this data so that we can have this structure carry ``messages'' across these neighbourhoods. What we mean by ``messages'' will be more rigorously defined in this chapter. The next definition is to marry the ideas defined above.

 \begin{tcolorbox}
 [width=\linewidth, sharp corners=all, colback=white!95!black]
\begin{definition}[Topological space]
\label{def:topospace}
A pair $(S,\mathcal{N})$ consisting of a nonempty set $S$ 
and a neighbourhood topology $\mathcal{N}$ on $S$ is called a \textbf{topological space}.
\end{definition}
\end{tcolorbox}

Now we define some topological objects and build up to combinatorial complexes. Let us start with the most familiar one.

\begin{tcolorbox}
[width=\linewidth, sharp corners=all, colback=white!95!black]
\begin{definition}[Graph]
\label{graph:main}
A \textbf{graph} on a nonempty set $V$ (called the set of vertices) is a pair $(V,E)$, where $E$ is a subset of $\mathcal{P}_2(V)$, and $\mathcal{P}_2(V)$ denotes the set of all 2-element subsets of $V$. Elements of $E$ are called \textbf{edges}.
\end{definition}
\end{tcolorbox}

\begin{example}
\label{graph:example}
let $V = \{0, 1, 2, \ldots, 9\}$ and

$E = \{\{1, 2\}, \{1, 6\}, \{1, 7\}, \{1, 8\}, \{2, 3\}, \{2, 6\}, \{3, 4\}, \{4, 5\}, \{6, 7\}, \{7, 8\}, \{8, 9\}\}$. Then this graph \(G\) is a pair \(G = (V, E)\). Figure~\ref{fig:graph} is a visual representation of G.

\end{example}

\begin{figure}[htbp]
\centering
\includegraphics[width = 0.8\hsize]{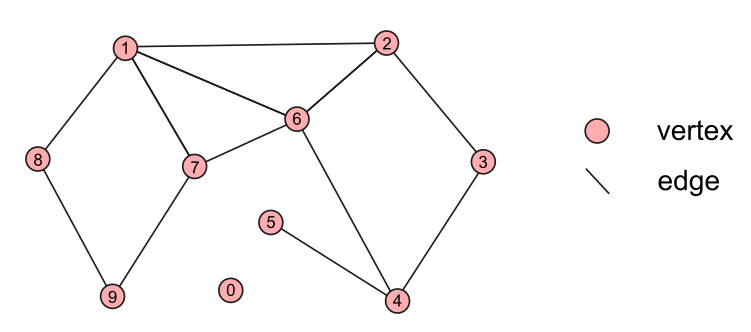}
\caption{A visual representation of an example graph in Example~\ref{graph:example}.}
\label{fig:graph}
\end{figure}

\begin{tcolorbox}
[width=\linewidth, sharp corners=all, colback=white!95!black]
\begin{definition}[Hypergraph]
\label{hyperG:main}
A \textbf{hypergraph} on a nonempty set $S$ is a pair $ (S,\mathcal{X})$, where $\mathcal{X}$ is a subset of $\mathcal{P}(S)\setminus\{\emptyset\}$. Elements of $\mathcal{X}$ are called \textbf{hyperedges}.
% {\color{red} tamal: calling the same thing with two names can be confusing, maybe just settle % on hyperedges and comment subsets of cardinality 2 are called edges?}
\end{definition}
\end{tcolorbox}

\begin{remark}
    From Definition~\ref{hyperG:main}, one can see graphs are simply a special case of hypergraphs where each hyperedge has exactly 2 elements. Therefore, hypergraphs generalise graphs and extend graphs which only include binary relationships between the vertices. Notice that hypergraphs allow hyperedges to be subsets of other hyperedges.
\end{remark}

\begin{example}
    \label{hypergraph:example}
    let $S = V = \{0, 1, 2, \ldots, 9\}$ and $\CCX = \{ \{1, 2\}, \{2, 3\}, \{2, 6\}, \{3, 4\}, \{4, 5\}, \{1,7,8,9\}, \\ \{0,4,5,6,7,8,9\}\}$. Then  $ (S,\mathcal{X})$ defines a hypergraph. Notice that this example hypergraph does not include all the edges of the graph defined above. This is simply because drawing all the binary hyperedges resulted in the visualisation that is too crowded with blue blobs. Figure~\ref{fig:hypergraph} is a visual representation of this hypergraph. 
\end{example}

\begin{figure}[htbp]
\centering
\includegraphics[width = 0.8\hsize]{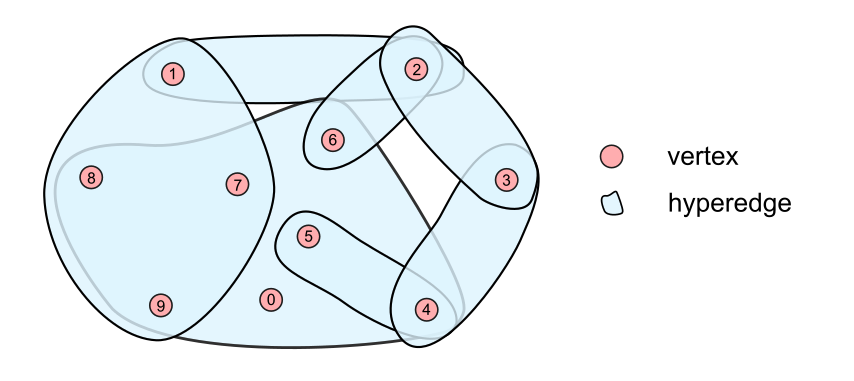}
\caption{A visual representation of an example graph in Example~\ref{hypergraph:example}.}
\label{fig:hypergraph}
\end{figure}

\begin{tcolorbox}
[width=\linewidth, sharp corners=all, colback=white!95!black]
\begin{definition}[Simplex]
\label{simplex:main}
Let $S$ be a nonempty set. A \textbf{simplex} on $S$ is a non-empty subset of $S$. The dimension of a simplex is one less than the cardinality of the subset. Specifically, if $\sigma \subseteq S$ and $\sigma \neq \emptyset$, then the dimension of $\sigma$, denoted as $\dim(\sigma)$, is given by $\dim(\sigma) = |\sigma| - 1$, where $|\sigma|$ represents the cardinality of $\sigma$.
\end{definition}
\end{tcolorbox}

\begin{tcolorbox}
[width=\linewidth, sharp corners=all, colback=white!95!black]
\begin{definition}[\href{https://app.vectary.com/p/4HZRioKH7lZ2jWESIBrjhf}{Simplicial complex}]
\label{SCs:main}

An \textbf{abstract simplicial complex}
on a nonempty set $S$ is a pair $(S,\mathcal{X})$, where $\mathcal{X}$ is a subset of $\mathcal{P}(S) \setminus \{\emptyset\}$ such that $ x \in \CCX$ and $y \subseteq x $ imply $y \in \CCX$. 
\end{definition}
\end{tcolorbox}
    Simplicial complexes also generalise graphs. To see this, one can imagine a simplicial complex with only 1-D and 2-D simplices, which are vertices and edges. However, they do not generalise hypergraphs as they come with the additional constraint. Although not immediately obvious, the constraint ``$ x \in \CCX$ and $y \subseteq x $ imply $y \in \CCX$" ensures that every element in $\CCX$ is a simplex. To see this, notice that $y$ in the above definition refers to an arbitrary subset of $x$ regardless of cardinality. For instance, if $x = \{ a, b, c, d\}$, ie. a tetrahedron, we must have all faces of the tetrahedron $\{ a, b, c\}$, $\{ a, b, d\}$, $\{ a, c, d\}$, and $\{b, c, d\}$ $\in$ $\CCX$. This constraint allows a sense of ``dimension'' or ``hierarchy'' to be contained in each simplex by (cardinality of the simplex) - $1$. In a hypergraph, the cardinality does not imply any sense of dimension or hierarchy.

\begin{example}
\label{simplex:example}
let $S$ and $E$ be the same as in Example~\ref{graph:example}, and let $\CCX = S \cup E \cup \{\{ 1,2,6\}\}$, then $(S,\CCX)$ defines a simplicial complex. Figure~\ref{fig:simplex} is a visual representation of this simplicial complex. For an interactive example, click on \textbf{(Simplicial Complex)} in the definition.

\end{example}

\begin{figure}[htbp]
\centering
\includegraphics[width = 0.8\hsize]{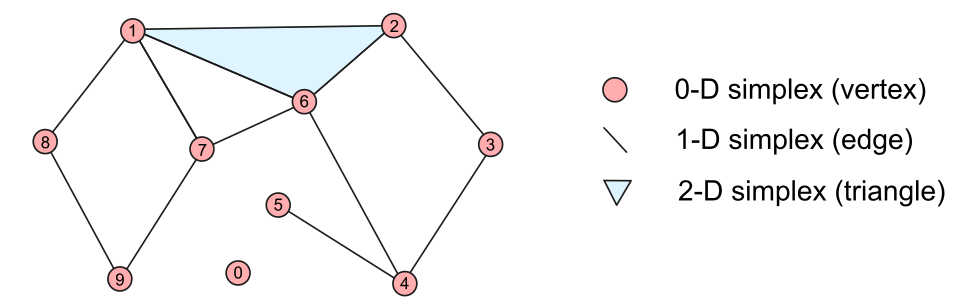}
\caption{A visual representation of the simplicial complex in Example~\ref{simplex:example}.}
\label{fig:simplex}
\end{figure}

\begin{tcolorbox}
[width=\linewidth, sharp corners=all, colback=white!95!black]
\begin{definition}[\href{https://app.vectary.com/p/3EBiRiJcYjFNvkbbWszQ0Z}{Cell Complex}]
\label{RCC:main}
A \textbf{cell/cellular complex} is a topological space $S$ with a partition into subspaces (\textbf{cells}) $\{x_\alpha\}_{\alpha \in P_{S} }$, where $P_{S}$ is an index set, satisfying the following conditions:
\begin{enumerate}
\item $S= \cup_{\alpha \in P_{S}} \Int(x_{\alpha})$,
where $\Int(x)$ denotes the interior of cell $x$.
\item For each $\alpha \in P_S$,
there exists a homeomorphism\footnotemark, called an \textbf{attaching map}, 
from $x_\alpha$ to $\R^{n_\alpha}$ for some $n_\alpha\in \N$,
called the \textbf{dimension} $n_\alpha$ of cell $x_\alpha$.
\item For each cell $x_{\alpha}$,
the boundary $\partial x_{\alpha}$ is a union of finitely many cells,
each having dimension less than the dimension of $x_{\alpha}$. 
\end{enumerate}
\end{definition}
\end{tcolorbox}

\footnotetext{We will not mathematically define ``homeomorphism'' in this report as we won't need it. However, one can intuitively think of a ``homeomorphism'' as a ``stretchy, bendy" transformation between two spaces that allows reshaping one into the other without tearing, cutting, or glueing. For instance, one can stretch and bend a sphere without creating any holes to form a cube. Thus, a cube and a sphere are homeomorphic to each other. However, a sphere and a torus (the shape of a donut) are not homeomorphic as one would have to create a hole whilst stretching to form the torus.}

\begin{remark}
    In standard literature, the above definition is not a ``cell complex''  but a ``regular cell complex''.
\end{remark} 

Cell complexes generalise graphs and simplicial complexes. You may click on the link on \href{https://app.vectary.com/p/3EBiRiJcYjFNvkbbWszQ0Z}{\textbf{Cell Complex}} to view an interactive example of a cell complex. We provide another example below as Example~\ref{cell:example}. Intuitively, a cell complex is a disjoint union of cells, with each of these cells being homeomorphic\footnotemark\textbf{} to the interior of a $k$ dimensional ball in $\R^k$. These cells are hierarchically attached to form higher-order cells. Thus, all the information regarding all the attachments of a cell complex can be stored in a sequence of incidence matrices defined in Definition~\ref {def:inc_mat}. The following are boundary examples:
The boundary of a 2-cell (a rectangle, for example) consists of three 1-cells (its edges).
The boundary of a 1-cell (an edge) consists of two 0-cells (its endpoints). Cell complexes generalise simplicial complexes, as they lift the condition that the cells must be simplices. However, like simplicial complexes, cell complexes do not generalise hypergraphs, because of the boundary condition. For instance, faces of a 2-Cell (say, a pentagon) cannot solely be a combination of 0-Cells, as condition 3 requires that the boundary be a union of \textbf{finitely} many cells. For an interactive example, click on \textbf{(Cell Complex)} in the definition.

\begin{example}
\label{cell:example}
let $S$ and $E$ be the same as in Example~\ref{graph:example}, and let $\{x_\alpha\}_{\alpha \in P_{S} } = S \cup E \cup \{\{ 1,2,6\}, \\ \{1,7,8,9\}\}$, then $(S,\CCX)$ defines a cellular complex. Figure~\ref{fig:cell} is a visual representation of this cellular complex. 
\end{example}

\begin{figure}[htbp]
\centering
\includegraphics[width = 0.8\hsize]{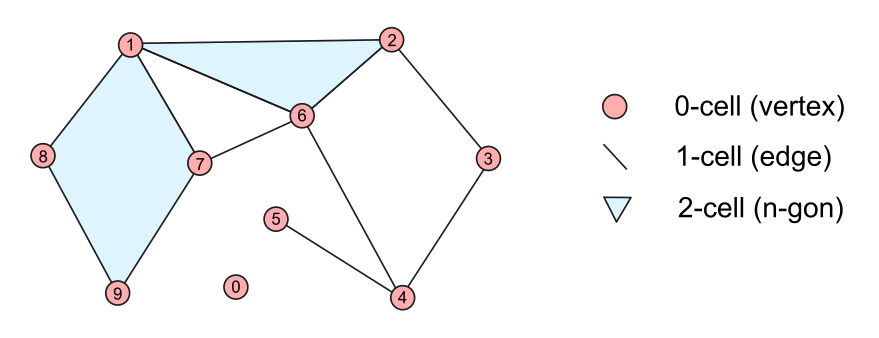}
\caption{A visual representation of the cellular complex in Example~\ref{cell:example}.}
\label{fig:cell}
\end{figure}

\section{Combinatorial Complexes}
\label{ccs}
\subsection{Definition}

\begin{tcolorbox}
[width=\linewidth, sharp corners=all, colback=white!95!black]
\begin{definition}[Rank function]\label{def:rank}
A \textbf{rank function} on a higher-order domain $\mathcal{X}$ is an order-preserving function
$\rk\colon \CCX\to \Znon$; i.e., $x\subseteq  y$ implies $\rk(x) \leq \rk(y)$ for all $x,y\in\CCX$. 
\end{definition}
\end{tcolorbox}

\begin{tcolorbox}
[width=\linewidth, sharp corners=all, colback=white!95!black]
\begin{definition}[CC]
\label{def:cc} A \textbf{combinatorial complex (CC)} is a triple $(S, \CCX, \rk)$ consisting of a set $S$, a subset $\CCX$ of $\mathcal{P}(S)\setminus\{\emptyset\}$, and a function
$\rk \colon \CCX\to \Znon$ with the following properties:
%that $x\subseteq  y$ implies $\rk(x) \leq \rk(y)$ for all $x,y\in\CCX$. 
	 \begin{enumerate}
	 \item for all $s\in S$, $\{s\}\in\CCX$, and
	 \item the function $\rk$ is order-preserving,
  which means that if $x,y\in \CCX$ satisfy $x\subseteq  y$, then $\rk(x) \leq \rk(y)$.
	 \end{enumerate}
  Elements of $S$ are called \textbf{entities} or \textbf{vertices}, elements of $\CCX$ are called \textbf{relations} or \textbf{cells}, and $\rk$ is called the \textbf{rank function} of the CC.
\end{definition}
\end{tcolorbox}

The rank of a cell $x\in\mathcal{X}$ is
the value $\rk(x)$ of the rank function $\rk$ at $x$.
The \emph{dimension} $\dim(\CCX)$ of a CC $\CCX$
is the maximal rank among its cells. A cell of rank $k$ is called a \emph{$k$-cell} and is denoted by $x^k$. The \textit{$k$-skeleton} of a CC $\CCX$, denoted $\CCX^{(k)}$, is the set of cells of rank at most $k$ in $\CCX$. The set of cells of rank exactly $k$ is denoted by $\mathcal{X}^k$. Note that this set corresponds to $\CCX^k=\rk^{-1}(\{k\})$. The $1$-cells are called the \emph{edges} of $\CCX$.
In general, an edge of a CC may contain more than two nodes. CCs whose edges have exactly two nodes are called \emph{graph-based} CCs. In this report, we only work with graph-based CCs.

    Combinatorial complexes generalise over graphs, hypergraphs, simplicial complexes, and cell complexes. Combinatorial complexes are similar to hypergraphs in the sense that a cell can have any cardinality, and similar to cell complexes in the sense that they exhibit a hierarchical structure explicitly defined by the rank function. Below is an example to show a CC which cannot be expressed as either of the aforementioned objects. 

\begin{example}
\label{cc:example}
let $S$ and $E$ be the same as in Example~\ref{graph:example}, and let $\CCX = E \cup \{\{ 1,2,6\}, \{1,7,8,9\}, \\ \{0, 4, 5, 6, 7, 8, 9\}\}$. Define $\rk$ by $\rk(\{ 1,2,6\}) = \rk(\{1,7,8,9\}) = 2$, $\rk(\{0, 4, 5, 6, 7, 8, 9\}) = 3$; $\forall x \in E$, $\rk(x) = 1$; $\forall x \in S$, $\rk(x) = 0$. Then $(S,\CCX, \rk)$ defines a combinatorial complex. Figure~\ref{fig:cc} is a visual representation of this combinatorial complex. Notice that the 3-cell can not be represented by a cell complex cell as it contains a 0-cell that is not contained by any 1-cell.
\end{example}

\begin{figure}[htbp]
\centering
\includegraphics[width = 0.8\hsize]{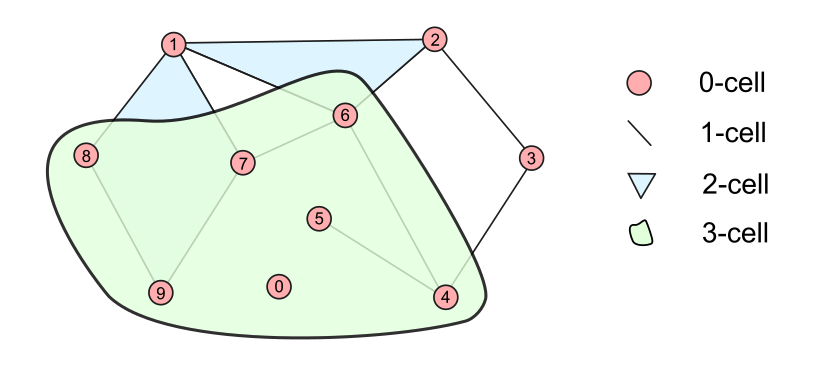}
\caption{A visual representation of the cellular complex in Example~\ref{cc:example}.}
\label{fig:cc}
\end{figure}

\subsection{Combinatorial Complex Neighbourhood}

\begin{tcolorbox}
[width=\linewidth, sharp corners=all, colback=white!95!black]
\begin{definition}[Sub-CC]
\label{sub-cc}
Let $(S,\CCX, \rk)$ be a CC.
A \textbf{sub-combinatorial complex (sub-CC)} of a CC $(S,\CCX, \rk)$ is a CC $(A,\mathcal{Y},\rk^{\prime})$ such that $A\subseteq S$, $\mathcal{Y}\subseteq\mathcal{X}$ and
$\rk^{\prime} = \rk|_{\mathcal{Y}}$ is the restriction of $\rk$ on $\mathcal{Y}$.
\end{definition}
\end{tcolorbox}

\begin{tcolorbox}
[width=\linewidth, sharp corners=all, colback=white!95!black]
\begin{definition}[CC-neighbourhood function]
\label{neighbourhood_structure}
A \textbf{CC-neighbourhood function} on a CC $(S,\CCX, \rk)$
is a function $\mathcal{N}$ that assigns to every sub-CC
$(A,\mathcal{Y},\rk^{\prime})$ of the CC
a nonempty collection $\mathcal{N}(\mathcal{Y})$ of subsets of $S$.
\end{definition}
\end{tcolorbox}

A neighbourhood function denotes some sense of vicinity of a given cell. How this neighbourhood can be defined is very general and context specific. Intuitively, in this work, we define neighbourhood functions in a sense of connectedness. Since we work with graph based CCs, one example of a neighbourhood of a node could be the set of the node itself along with the other nodes connected to this node via edges. We will define slightly more complex neighbourhood functions via adjacency and incidence neighbourhood functions with Definitions~\ref{filtered_incidence}, and \ref{def:k_coa_nf}.

\begin{tcolorbox}
[width=\linewidth, sharp corners=all, colback=white!95!black]
\begin{definition}[Neighbourhood matrix]
\label{neighbourhood_matrix}
Let $\mathcal{N}$ be a neighbourhood function defined on a CC $\CCX$.
Moreover, let $\mathcal{Y}=\{y_1,\ldots,y_n\}$ and $\mathcal{Z}=\{z_1,\ldots,z_m\}$
be two collections of cells in $\CCX$ such that $\mathcal{N}(y_{i}) \subseteq \CCZ$
for all $1\leq i \leq n $.
The \textbf{neighbourhood matrix
of $\mathcal{N}$ with respect to $\mathcal{Y}$ and $\mathcal{Z}$} is the
$|\CCZ| \times |\CCY|$ binary matrix $G$ whose 
$(i,j)$-th entry $[G]_{ij}$ has value $1$ if $z_i \in \mathcal{N}(y_j)$
and $0$ otherwise.
\end{definition}
\end{tcolorbox}

In practice and throughout this report,  $\CCY$ and $\CCZ$ will be $\CCX^k$ and $\CCX^l$ for some ranks $k,l$ of a $\CC$.

\begin{tcolorbox}
[width=\linewidth, sharp corners=all, colback=white!95!black]
\begin{definition}[$k$-down/up incidence neighbourhood functions]
\label{filtered_incidence}
Let $(S,\CCX, \rk)$ be a CC.
For any $k\in\N$,
the \textbf{$k$-down incidence neighbourhood function}
$\mathcal{N}_{\searrow,k}(x)$ of a cell $x\in\mathcal{X}$
is defined to be the set
$\{ y\in \CCX \mid y \subsetneq x, \rk(y)=\rk(x)-k \}$.
The \textbf{$k$-up incidence neighbourhood function}
$\mathcal{N}_{\nearrow,k}(x)$ of $x$
is defined to be the set
$\{ y\in \CCX \mid x \subsetneq y, \rk(y)=\rk(x)+k \}$.
\end{definition}
\end{tcolorbox}

Intuitively, k-down incidence neighbourhood of a cell $x$ is the set of all the cells contained in $x$ with a specified rank ($\rk(x)-k$). k-up incidence neighbourhood of a cell $x$ is the set of all the cells that contain $x$ with a specified rank ($\rk(x)+k$).

\begin{tcolorbox}
[width=\linewidth, sharp corners=all, colback=white!95!black]
\begin{definition}[(co)-incidence matrix]
\label{def:inc_mat}
Let $(S,\CCX, \rk)$ be a CC.
For any $r,k \in \Znon$ with $0\leq r<k \leq \dim(\CCX) $, the
\textbf{$(r,k)$-incidence matrix} 
$B_{r,k}$ between $\CCX^{r}$ and $\CCX^{k}$
is defined to be the $ |\CCX^r| \times |\CCX^k|$ binary matrix
whose $(i, j)$-th entry $[B_{r,k}]_{ij}$
equals one if $x^r_i$ is incident to $x^k_j$
and zero otherwise. In this work we refer to $coB_{r,k} = B_{r,k}^{\top}$ as the
\textbf{$(r,k)$-co-incidence matrix} between $\CCX^{r}$ and $\CCX^{k}$.
\end{definition}
\end{tcolorbox}

\begin{remark}
    Co-incidence matrices are not originally defined in \cite{hajij2023topological}.
\end{remark}

Co-incidence matrices are very intuitive from a data science perspective. One can think of every cell of rank $k$ in a CC as a row of this matrix encoded by which $r$ cells they contain. When $r = 0$, it is straightforward to see that there is a bijection between $\CCX^k$ and the rows of $coB_{r,k}$. This is similar to how we usually represent a dataset as a matrix where every row represents a data point. Co-incidence matrices will be crucial for our model presented in this report. We will elaborate further in Section~\ref{representation}.

\begin{tcolorbox}
[width=\linewidth, sharp corners=all, colback=white!95!black]
\begin{definition}[$k$-(co)adjacency neighbourhood functions]
\label{def:k_coa_nf}
Let $(S, \CCX, \rk)$ be a CC.
For any $k\in\N$,
The \textbf{$k$-adjacency neighbourhood function} $\mathcal{N}_{a,k}(x)$

of a cell $x \in \CCX$ is defined to be the set
\begin{equation*}
% \label{adj}
\{ y \in \CCX \mid \rk(y)=\rk(x),
\exists z \in \CCX
\text{ with } \rk(z)=\rk(x)+k \text{ such that } x,y\subsetneq z \}.
\end{equation*}
The \textbf{$k$-coadjacency neighbourhood function}
$\mathcal{N}_{co,k}(x)$ of $x$ is defined to be the set
\begin{equation*}
% \label{coadj}
 \{ y \in \CCX \mid \rk(y)=\rk(x),
\exists z \in \CCX
\text{ with } \rk(z)=\rk(x)-k 
\text{ such that } z\subsetneq y\text{ and }z\subsetneq x \}.
\end{equation*}
\end{definition}
\end{tcolorbox}

In words, k-adjacency neighbourhood of a cell $x$ is the set of all the cells with rank same as $x$, and contained in cells with a specified rank ($\rk(x)+k$) which also contain $x$. In a graph-based CC, this may be all the nodes a node is connected to, including the node itself. K-coadjacency neighbourhood of a cell $x$ is the set of all the cells with rank same as $x$, and contain cells with a specified rank ($\rk(x)-k$) which $x$ also contains. In a graph-based CC, this may be all the edges which share a node with a given edge, including the edge itself.

\begin{tcolorbox}
[width=\linewidth, sharp corners=all, colback=white!95!black]
\begin{definition}[Adjacency matrix]
\label{def:adj}
For any $r\in\Znon$ and $k\in \mathbb{Z}_{>0} $ with $0\leq r<r+k \leq \dim(\CCX) $,
the \textbf{$(r,k)$-adjacency matrix} $A_{r,k}$
among the cells of $\CCX^{r}$
with respect to the cells of $\CCX^{k}$ 
is defined to be the $ |\CCX^r| \times |\CCX^r|$ binary matrix
whose $(i, j)$-th entry $[A_{r,k}]_{ij}$ equals one
if $x^r_i$ is $k$-adjacent to $x^r_j$ and zero otherwise.
\end{definition}
\end{tcolorbox}

An adjacency matrix is by far the most common approach to represent a graph, since it encodes all the information in a graph (potentially directed with self loops) without node or edge features in a single square matrix. This is the approach taken in \cite{clarkson_damnets_2023}, \cite{dai_scalable_2020}, \cite{inbook}.

\begin{tcolorbox}
[width=\linewidth, sharp corners=all, colback=white!95!black]
\begin{definition}[Coadjacency matrix]
\label{def:coadj}
For any $r\in \Znon$ and $k\in\N$ with $0\leq r-k<r \leq \dim(\CCX) $, the \textbf{$(r,k)$-coadjacency matrix} $coA_{r,k}$
among the cells of $\CCX^{r}$
with respect to the cells of $\CCX^{k}$ 
is defined to be the $ |\CCX^r| \times |\CCX^r|$ binary matrix
whose $(i, j)$-th entry $[coA_{r,k}]_{ij}$ has value $1$
if $x^r_i$ is $k$-coadjacent to $x^r_j$ and $0$ otherwise.
\end{definition}
\end{tcolorbox}

\begin{remark}
    Unlike the case with incidence matrices, $A_{r,k}^{\top} \neq coA_{r,k}$.
\end{remark}

\subsection{Data Defined on Combinatorial Complexes}

\begin{tcolorbox}
[width=\linewidth, sharp corners=all, colback=white!95!black]
\begin{definition}[k-cochain spaces]
Let $\mathcal{C}^k(\CCX,\R^d )$
be the $\R$-vector space of functions
$\mathbf{H}_k\colon\CCX^k\to \R^d$ for a rank $k \in \Znon$ and dimension $d$.
$d$ is called the \textbf{data dimension}.
$\mathcal{C}^k(\CCX,\R^d)$ is called the \textbf{$k$-cochain space}.
Elements $\mathbf{H}_k$ in $\mathcal{C}^k(\CCX,\R^d)$
are called \textbf{$k$-cochains} or \textbf{$k$-signals}. 
\end{definition}
\end{tcolorbox}

K-cochain spaces are used to define signals on cells of a CC. Intuitively, one can think of $\mathbf{H}_k$ as the features (or signals) defined on the $k$-cells of a CC. For instance, given a graph based CC, to define the features of nodes, we could use the $0$-cochain $\mathbf{H}_0$ to map every node to a vector of length $d$. Given a node ordering, we can explicitly represent $\mathbf{H}_0$ by a feature matrix $\mathbf{F_0} \in \R^{|\CCX^0| \times d}$, where $i$'th row is the horizontal feature vector $\mathbf{h}_{x_{i}^{0}}$ corresponding to the features defined on node $x_{i}^{0}$. For the rest of this report, unless stated otherwise, $\mathbf{H}_k = \mathbf{F_k} \in \R^{|\CCX^0| \times d}$.

\begin{tcolorbox}
[width=\linewidth, sharp corners=all, colback=white!95!black]
\begin{definition}[Cochain maps]
For $r< k$, an incidence matrix $B_{r,k}$ induces a map
\begin{align*}
    B_{r,k}\colon \mathcal{C}^k(\CCX) &\to   \mathcal{C}^r(\CCX),\\
    \mathbf{H}_k &\to  B_{r,k}(\mathbf{H}_k),
\end{align*}
where $B_{r,k}(\mathbf{H}_k)$ denotes the usual product $B_{r,k} \mathbf{H}_k$
of matrix $B_{r,k}$ with matrix $\mathbf{H}_k$.
Similarly, an $(r,k)$-adjacency matrix $A_{r,k}$ induces a map
\begin{align*}
    A_{r,k}\colon \mathcal{C}^r(\CCX) &\to   \mathcal{C}^r(\CCX),\\
    \mathbf{H}_r &\to  A_{r,k}(\mathbf{H}_r).
\end{align*}
\end{definition}
These two types of maps between cochain spaces are called \textbf{cochain maps}.
\end{tcolorbox}

Cochain maps are the main building blocks of higher-order message passing in Definition~\ref{homp-definition}. They are used to redistribute the signal defined on $k$-cells to signals defined on $r$-cells, which one can intuitively think of passing messages across the cells of a CC.

\subsection{Construction of higher-order networks}

Below we define the general form of \textbf{Combinatorial Complex Neural Networks} (CCNNs), of which \textbf{Combinatorial Complex Attention Neural Networks} (CCANNs) are a special case of. Note that the below definition is very general and does not give any information regarding the computations of these networks.

\begin{tcolorbox}
[width=\linewidth, sharp corners=all, colback=white!95!black]
\begin{definition}[CCNN]
\label{hoans_definition}
Let $\CCX$ be a CC.
Let $\mathcal{C}^{i_1}\times\mathcal{C}^{i_2}\times \ldots \times  \mathcal{C}^{i_m}$ and $\mathcal{C}^{j_1}\times\mathcal{C}^{j_2}\times \ldots \times  \mathcal{C}^{j_n}$ be a Cartesian product of $m$ and $n$ cochain spaces defined on $\CCX$. A \textbf{combinatorial complex neural network (CCNN)} 
is a function of the form
\begin{equation*} 
% \label{tensornets}
\CCN: \mathcal{C}^{i_1}\times\mathcal{C}^{i_2}\times \ldots \times  \mathcal{C}^{i_m} \longrightarrow \mathcal{C}^{j_1}\times\mathcal{C}^{j_2}\times \ldots \times \mathcal{C}^{j_n}.
\end{equation*}
\end{definition}
\end{tcolorbox}

Below is the most important definition in this chapter:

\begin{tcolorbox}
[width=\linewidth, sharp corners=all, colback=white!95!black]
\begin{definition}[Cochain push-forward]
\label{pushing_exact_definition}
Consider a CC $\CCX$,
a cochain map $G\colon\mathcal{C}^i(\CCX)\to \mathcal{C}^j(\CCX)$,
and a cochain $\mathbf{H}_i$ in $\mathcal{C}^i(\CCX)$.
A \textbf{(cochain) push-forward} induced by $G$ is an operator
$\mathcal{F}_G \colon \mathcal{C}^i(\CCX)\to \mathcal{C}^j(\CCX)$ 
defined via
    \begin{equation}
    \mathcal{F}_G(\mathbf{H}_i) =  \mathbf{K}_j=[ \mathbf{k}_{y^j_1},\ldots,\mathbf{k}_{y^j_{|\CCX^j|} }], 
\end{equation}
such that for $k=1,\ldots,|\CCX^j|$,
    \begin{equation}
    \label{functional}
    \mathbf{k}_{y_k^j}= \bigoplus_{x_l^i \in \mathcal{N}_{G^T}(y_k^j)} \alpha_{G} ( \mathbf{ \mathbf{h}}_{x_l^i}  ),
    \end{equation}
where $\bigoplus$ is a permutation-invariant aggregation function and $\alpha_G$ is a differentiable function. 
\end{definition}
\end{tcolorbox}

Let us unpack Definition~\ref{pushing_exact_definition} using an example: 

\begin{example}
\label{example:push-forward}
Let us take the simple case of the graph based CC in Example~\ref{cc:example}. Let the cochain map be $B_{0,1}$.
In this scenario, we would like to take the cochains defined on the $1$-cells, and given some rule, assign some cochains to every $0$-cell in our CC. For simplicity, let $H_1 = \mathbf{I}$ be the features defined on the edges, where edge indexed by alphabetical order, i.e. every edge is one-hot encoded and indexed as follows:

Edge 1: \{1,2\}, Edge 2: \{1,6\}, Edge 3: \{1,7\}, Edge 4: \{1,8\},
Edge 5: \{2,3\}, Edge 6: \{2,6\}, Edge 7: \{3,4\}, Edge 8: \{4,5\},
Edge 9: \{6,7\}, Edge 10: \{7,8\}, Edge 11: \{8,9\}.

Let $\alpha_{B_{0,1}}(\mathbf{h}) = \mathbf{h}$. Let $\bigoplus$ simply be the operation of addition. Under this formulation, let us focus on node $3$: $y^{0}_3$. $y^{0}_3$ has only two neighbours in the set $\mathcal{N}_{B_{0,1}(y^{0}_3)} = \{\{2,3\}, \{3,4\}\}$ with features of the edges according to the feature matrix \( F = H_1\) as follows:

\begin{itemize}
    \item \textbf{Edge 5}: \( F[5, :] = \begin{pmatrix} \begin{array}{ccccccccccc} 0 & 0 & 0 & 0 & 1 & 0 & 0 & 0 & 0 & 0 & 0 \end{array} \end{pmatrix} \)
    
    \item \textbf{Edge 7}: \( F[7, :] = \begin{pmatrix} \begin{array}{ccccccccccc} 0 & 0 & 0 & 0 & 0 & 0 & 1 & 0 & 0 & 0 & 0 \end{array} \end{pmatrix} \)
\end{itemize}

Putting it all together, we have:

\begin{equation}
\mathbf{k}_{y_3^0} = \bigoplus_{x_l^1 \in \mathcal{N}_{B_{0,1}(y^{0}_3)}} \alpha_{B_{0,1}} ( \mathbf{ \mathbf{h}}_{x_l^1}  ) =  F[5, :] + F[7, :] 
= \begin{pmatrix} \begin{array}{ccccccccccc} 0 & 0 & 0 & 0 & 1 & 0 & 1 & 0 & 0 & 0 & 0 \end{array} \end{pmatrix}.
\end{equation}
We repeat this for every node in the CC and we have successfully defined $\mathcal{F}_{B_{0,1}}$. Note that the above verbose procedure is equivalent to $\mathcal{F}_{B_{0,1}}(\mathbf{H_1}) = B_{0,1}\mathbf{H_1}$.

\end{example}

%  For simplicity, let $H_1 = \mathbf{1}$ be the features defined on the edges, i.e. $d = 1$ and every edge's cochain is a vector of length $1$, of which the only element is the integer $1$. Let $\bigoplus$ simply be addition and 
% \end{example}

\begin{tcolorbox}
[width=\linewidth, sharp corners=all, colback=white!95!black]
\begin{definition}[Merge node]
\label{exact_definition_merge_node}
Let $\CCX$ be a CC.
Moreover, let $G_1\colon\mathcal{C}^{i_1}(\mathcal{X})\to\mathcal{C}^j(\mathcal{X})$ and
$G_2\colon\mathcal{C}^{i_2}(\mathcal{X})\to\mathcal{C}^j(\mathcal{X})$ be two cochain maps.
Given a cochain vector $(\mathbf{H}_{i_1},\mathbf{H}_{i_2}) \in \mathcal{C}^{i_1}\times \mathcal{C}^{i_2}$, a \textbf{merge node} $\mathcal{M}_{G_1,G_2}\colon\mathcal{C}^{i_1} \times \mathcal{C}^{i_2} \to \mathcal{C}^j$ is defined as
\begin{equation}
    \mathcal{M}_{G_1,G_2}(\mathbf{H}_{i_1},\mathbf{H}_{i_2})= \beta\left( \mathcal{F}_{G_1}(\mathbf{H}_{i_1})  \bigotimes \mathcal{F}_{G_2}(\mathbf{H}_{i_2}) \right),
\end{equation}
where $\bigotimes \colon \mathcal{C}^j \times \mathcal{C}^j \to \mathcal{C}^j $ is an aggregation function, $\mathcal{F}_{G_1}$ and $\mathcal{F}_{G_2}$ are push-forward operators induced by $G_1$ and $G_2$, and $\beta$ is an activation function.
\end{definition}
\end{tcolorbox}

\begin{tcolorbox}
[width=\linewidth, sharp corners=all, colback=white!95!black]
\begin{definition} [Split node]
 \label{exact_definition_split_node}
Let $\CCX$ be a CC.
Moreover, let $G_1\colon\mathcal{C}^{j}(\mathcal{X})\to\mathcal{C}^{i_1}(\mathcal{X})$
and $G_2\colon\mathcal{C}^{j}(\mathcal{X})\to\mathcal{C}^{i_2}(\mathcal{X})$ be two cochain maps. Given a cochain $\mathbf{H}_{j} \in \mathcal{C}^{j}$, a \textbf{split node} $\mathcal{S}_{G_1,G_2}\colon\mathcal{C}^j \to \mathcal{C}^{i_1} \times \mathcal{C}^{i_2} $ is defined as
\begin{equation}
    \mathcal{S}_{G_1,G_2}(\mathbf{H}_{j})= \left(  \beta_1(\mathcal{F}_{G_1}(\mathbf{H}_{j})) , \beta_2(\mathcal{F}_{G_2}(\mathbf{H}_{j})) \right),
\end{equation}
where  $\mathcal{F}_{G_i}$ is a push-forward operator induced by $G_i$,
and $\beta_i$ is an activation function for $i=1, 2$.
\end{definition}
\end{tcolorbox}

Split node is simply a tuple of cochain push-forward operations. However, it is useful in helping us think intuitively in forming CCNNs. Push forward, merge node, and split node operations are collectively referred to as \textbf{elementary tensor operations}, as we can build any CCNNs by a combination of these operations.
\begin{tcolorbox}
[width=\linewidth, sharp corners=all, colback=white!95!black]
\begin{definition}[Higher-order attention]
\label{hoa}
Let $\CCX$ be a CC, $\mathcal{N}$
a neighbourhood function defined on $\CCX$, and $\CCY_0$ a sub-CC of $\CCX$.
Let $\mathcal{N}(\mathcal{Y}_0)=
\{ \CCY_1,\ldots, \CCY_{|\mathcal{N}(\CCY_0)|} \}$ be a set
of sub-CCs that are in the vicinity of $\mathcal{Y}_0$ with respect to the neighbourhood function $\mathcal{N}$.
A \textbf{higher-order attention}
of $\CCY_0$ with respect to $\mathcal{N}$
is a function
$a\colon {\CCY_0}\times \mathcal{N}(\CCY_0)\to [0,1]$
that assigns a weight $a(\CCY_0, \CCY_i)$
to each element $\CCY_i\in\mathcal{N}(\CCY_0)$
such that $\sum_{i=1}^{| \mathcal{N}(\CCY_0)|} a(\CCY_0,\CCY_i)=1$.     
\end{definition}
\end{tcolorbox}

Intuitively, higher-order attention $a$ decides how much weight to put to a neighbour of a given cell whilst computing the elementary tensor operations. In the following definition, we incorporate this idea into the push-forward operation.

\begin{tcolorbox}
[width=\linewidth, sharp corners=all, colback=white!95!black]
\begin{definition}[CC-attention push-forward for cells of equal rank]
	\label{hoan_sym}
	Let $G\colon C^{s}(\CCX)\to C^{s}(\CCX)$ be a neighbourhood matrix. A \textbf{CC-attention push-forward}
	induced by $G$ is a cochain map $\mathcal{F}^{att}_{G}\colon C^{s}(\CCX, \mathbb{R}^{d_{s_{in}}}) \to C^{s}(\CCX,\mathbb{R}^{d_{s_{out}}}) $ defined as
	\begin{equation}
	\label{attention1}
	\mathbf{H}_s \to \mathbf{K}_{s} = 
	(G \odot att)  \mathbf{H}_{s}  W_{s} ,
	\end{equation}
 where $\odot$ is the Hadamard product, $W_{s}  \in \mathbb{R}^{d_{s_{in}}\times d_{s_{out}}} $ are trainable
	parameters, and $att\colon C^{s}(\CCX)\to C^{s}(\CCX) $ is a \textbf{higher-order attention matrix} that has the same dimension as matrix $G$. The $(i,j)$-th entry of matrix $att$ is defined as
	\begin{equation}
	att(i,j) =  \frac{e_{ij}}{ \sum_{k \in \mathcal{N}_{G}(i)} e_{ik}  },    
	\end{equation}
	where $e_{ij}= \phi(a^T [W_{s} \mathbf{H}_{s,i}||W_{s} \mathbf{H}_{s,j} ] )$,
	$a \in \mathbb{R}^{2 \times s_{out}} $ is a trainable vector,  $[a ||b ]$ denotes the concatenation of $a$ and $b$, $\phi$ is an activation function, and $\mathcal{N}_{G}(i)$ is the neighbourhood of cell $i$ with respect to matrix $G$. 
\end{definition}
\end{tcolorbox}

$(G \odot att)$ in equation~\ref{attention1} defines probability a discrete distribution along each column: Namely, of a given feature, how much weight one is aught to give to each neighbour. $W_s$ is a learnable linear transformation applied to $\mathbf{H_s}$, allowing us to control the number of features for the target cochain space, explained in the original Graph Attenion Networks (GATs)\cite{velickovic_graph_2018}: ``In order to obtain sufficient expressive power to transform the input features into higher-level features, at least one learnable linear transformation is required''. Note the similarity of the definition above to the one in \cite{velickovic_graph_2018}. $a^T$ is the attention mechanism which ultimately decides on the weights, given the transformed feature vector of the cell that we compute the attentions from ($W_{s} \mathbf{H}_{s,i}$), and the transformed feature vector of the cell that we compute the attentions to ($W_{s} \mathbf{H}_{s,j}$). The following definition generalises Definition~\ref{attention1} to cells of different ranks:

\begin{tcolorbox}
[width=\linewidth, sharp corners=all, colback=white!95!black]
\begin{definition}[CC-attention push-forward for cells of unequal ranks]
\label{hoan_asym}
For $s\neq t$,
let $G\colon C^{s}(\CCX)\to C^{t}(\CCX)$ be a neighbourhood matrix. A \textbf{CC-attention block} induced by $G$ is a cochain map $\mathcal{F}_{G}^{att}  {\mathcal{A}}\colon C^{s}(\CCX,\mathbb{R}^{d_{s_{in}}}) \times C^{t}(\CCX,\mathbb{R}^{d_{t_{in}}}) \to C^{t}(\CCX,\mathbb{R}^{d_{t_{out}}}) \times C^{s}(\CCX,\mathbb{R}^{d_{s_{out}}}) $ defined as
	\begin{equation}
	(\mathbf{H}_{s},\mathbf{H}_{t}) \to  (\mathbf{K}_{t}, \mathbf{K}_{s} ), 
	\end{equation} 
	with
	\begin{equation}
	\label{attention2}
 %        \begin{aligned}
	% \mathbf{K}_{t} &=   ( G \odot att_{s\to t})  \mathbf{H}_{s} W_{s} ,\\
 %        \mathbf{K}_{s} &= (G^T \odot att_{t\to s})  \mathbf{H}_{t}  W_{t} , 
 %        \end{aligned}
	\mathbf{K}_{t} =   ( G \odot att_{s\to t})  \mathbf{H}_{s} W_{s} ,\;
        \mathbf{K}_{s} = (G^T \odot att_{t\to s})  \mathbf{H}_{t}  W_{t} , 
	\end{equation}
	where $W_s \in \mathbb{R}^{d_{s_{in}}\times d_{t_{out}}} , W_t \in \mathbb{R}^{d_{t_{in}}\times d_{s_{out}}} $ are
	trainable
	parameters, and $att^{s\to t}\colon C^{s}(\CCX)\to C^{t}(\CCX) , att^{t\to s}\colon C^{t}(\CCX)\to C^{s}(\CCX) $ are \textbf{higher-order attention matrices} that have the same dimensions as matrices $G$ and $G^T$, respectively. The $(i,j)$-th entries of matrices  $att^{s\to t}$ and $att^{t\to s}$ are defined as
	\begin{equation}
\label{eq:ast_ats}
% \begin{aligned}
% 	att_{s\to t}(i,j) &=  \frac{e_{ij}}{ \sum_{k \in \mathcal{N}_{G} (i) e_{ik} } },\\
%  att_{t\to s}(i,j) &=  \frac{f_{ij}}{ \sum_{k \in \mathcal{N}_{G^T} (i) f_{ik} } },
% \end{aligned}
	att^{s\to t}_{ij} =  \frac{e_{ij}}{ \sum_{k \in \mathcal{N}_{G} (i)} e_{ik}  },\;
 att^{t\to s}_{ij} =  \frac{f_{ij}}{ \sum_{k \in \mathcal{N}_{G^T} (i)} f_{ik}  },
	\end{equation} 
	with
	\begin{equation}
 \label{eq:ef}
	% \begin{aligned}
	% e_{ij} & = \phi((a)^T [W_s \mathbf{H}_{s,i}||W_t \mathbf{H}_{t,j}] ),\\
	% f_{ij} & = \phi((rev(a)^T [W_t \mathbf{H}_{t,i}||W_s \mathbf{H}_{s,j}]),
	% \end{aligned}
	e_{ij} = \phi((a)^T [W_s \mathbf{H}_{s,i}||W_t \mathbf{H}_{t,j}] ),\;
	f_{ij} = \phi(rev(a)^T [W_t \mathbf{H}_{t,i}||W_s \mathbf{H}_{s,j}]),
	\end{equation}
	where $a \in \mathbb{R}^{t_{out} + s_{out}} $ is a trainable vector, and
	$rev(a)= [a[t_{out}:] ||  a[:t_{out}]]$
\end{definition}
\end{tcolorbox}

\begin{remark}
    Note that the definition is pretty much the same as Definition~\ref{attention1} until equation~\ref{eq:ef}, where we replace the linearly transformed feature vector of the cell of the same rank with the respective linear transformed cell of different rank. Note that the attention mechanism $a$ is shared between $e_{ij}$ and $f_{ij}$.
\end{remark}
    
Next, define the computational framework that puts together everything in the preliminaries section and formalises the flow of information in the cells of a higher-order network. Note that we are not introducing a novel concept here but bring everything together.

\begin{tcolorbox}
[width=\linewidth, sharp corners=all, colback=white!95!black]
\begin{definition}[Higher-order message passing on a CC]
\label{homp-definition}
Let $\CCX$ be a CC.  Let $\mathcal{N}=\{ \mathcal{N}_1,\ldots,\mathcal{N}_n\}$ be a set of neighbourhood functions defined on  $\CCX$. Let $x$ be a cell and $y\in \mathcal{N}_k(x)$ for some $\mathcal{N}_k \in \mathcal{N}$. A \textbf{message} $m_{x,y}$ between cells $x$ and $y$
is a computation that depends on these two cells
or on the data supported on them.
Denote by $\mathcal{N}(x)$ the multi-set  $\{\!\!\{ \mathcal{N}_1(x) , \ldots ,  \mathcal{N}_n (x) \}\!\!\}$,
and by $\mathbf{h}_x^{(l)}$ some data supported on the cell $x$ at layer $l$. \textbf{Higher-order message passing} on $\CCX$, induced by $\mathcal{N}$, is defined via the following four update rules:
\begin{align}
\label{eqn:homp0}
m_{x,y} &= \alpha_{\mathcal{N}_k}(\mathbf{h}_x^{(l)},\mathbf{h}_y^{(l)}), \\ \label{eqn:homp1}
m_{x}^k &=  \bigoplus_{y \in \mathcal{N}_k(x)}  m_{x,y}, \; 1\leq k \leq n,   \\ \label{eqn:homp2}
m_{x} &=  \bigotimes_{ \mathcal{N}_k \in \mathcal{N} } m_x^k, \\ \label{eqn:homp3}
\mathbf{h}_x^{(l+1)} &= \beta (\mathbf{h}_x^{(l)}, m_x). 
\end{align}
Here, $\bigoplus$ is a permutation-invariant aggregation function
called the \textbf{intra-neighbourhood} of $x$,
$\bigotimes$ is an aggregation function
called the \textbf{inter-neighbourhood} of $x$,
and $\alpha_{\mathcal{N}_k},\beta$ are differentiable functions.
\end{definition}
\end{tcolorbox}

Next, we present a generalised case that incorporates the CC-attention push-forward mechanism introduced in Definitions~\ref{attention1}, \ref{attention2}:

\begin{tcolorbox}
[width=\linewidth, sharp corners=all, colback=white!95!black]
\begin{definition}[Attention higher-order message passing on a CC]
\label{attention_homp}
Let $\CCX$ be a CC.
Let $\mathcal{N}=\{ \mathcal{N}_1,\ldots,\mathcal{N}_n\}$ be a set of neighbourhood functions defined on  $\CCX$.
Let $x$ be a cell and $y\in \mathcal{N}_k(x)$ for some $\mathcal{N}_k \in \mathcal{N}$. A \textbf{message} $m_{x,y}$ between cells $x$ and $y$
is a computation that depends on these two cells, or the data supported on them.
Denote by $\mathcal{N}(x)$ the multi-set  $\{\!\!\{ \mathcal{N}_1(x) , \ldots ,  \mathcal{N}_n (x) \}\!\!\}$,
and by $\mathbf{h}_x^{(l)}$ some data supported on the cell $x$ at layer $l$.
\textbf{Attention higher-order message passing} on $\CCX$, 
induced by $\mathcal{N}$, is defined via the following four update rules:
\begin{align}
\label{ahomp0}
m_{x,y} &= \alpha_{\mathcal{N}_k}(\mathbf{h}_x^{(l)},\mathbf{h}_y^{(l)}), \\ \label{ahomp1}
m_{x}^k &=  \bigoplus_{y \in \mathcal{N}_k(x)} a^k(x,y)  m_{x,y}, \; 1\leq k \leq n ,   \\ \label{ahomp2}
m_{x} &=  \bigotimes_{ \mathcal{N}_k \in \mathcal{N} } b^k m_x^k ,\\ \label{ahomp3}
\mathbf{h}_x^{(l+1)} &= \beta (\mathbf{h}_x^{(l)}, m_x) .
\end{align}
Here, $a^k \colon \{x\} \times \mathcal{N}_k(x)\to [0,1] $ is a higher-order attention function  (Definition~\ref{hoa}),
$b^k$ are trainable attention weights satisfying $\sum_{k=1}^n b^k=1$,
$\bigoplus$ is a permutation-invariant aggregation function, $\bigotimes$ is an aggregation function,  $\alpha_{\mathcal{N}_k}$ and $\beta$ are differentiable functions.
\end{definition}
\end{tcolorbox}

\subsection{CC Representation}
\label{representation}

To incorporate a CC into an implementation of a CCNN, we need to efficiently represent CC in some form. This proves to be trickier than expected. Proposition 8.1 in \cite{hajij_topological_2023} asserts that we can represent a CC $(S, \CCX,\rk)$ via either one of the three:

\begin{enumerate}
\item Incidence matrices $\{B_{k,k+1}\}_{k=0}^{ \dim(\CCX) -1}$
\item Adjacency matrices $\{A_{k,1}\}_{k=0}^{\dim(\CCX)-1}$
\item Coadjacency matrices $\{coA_{k,1}\}_{k=1}^{\dim(\CCX)}$.
\end{enumerate}

This is the approach taken in \cite{carrel_combinatorial_2024}. However, we argue we can not represent a CC like that. To see this, we turn our attention to Example~\ref{cc:example}. Let us denote the CC in Example~\ref{cc:example} by $\CCX_1$. Now let us construct a $\CCX_2$ very similar to $\CCX_1$ by removing the element $5$ from the only $3$-cell. So we have the following $\CCX_2$:

$S$ and $E$ are the same as in Example~\ref{graph:example}, and let $\CCX = E \cup \{\{ 1,2,6\}, \{1,7,8,9\}, \\ \{0, 4, 6, 7, 8, 9\}\}$. Define $\rk$ by $\rk(\{ 1,2,6\}) = \rk(\{1,7,8,9\}) = 2$, $\rk(\{0, 4, 6, 7, 8, 9\}\}) = 3$; $\forall x \in E$, $\rk(x) = 1$; $\forall x \in S$, $\rk(x) = 0$. 

Clearly, $\CCX_1$ and $\CCX_2$ are different CCs, but they have the same $\{B_{k,k+1}\}_{k=0}^{ \dim(\CCX) -1}$, $\{A_{k,1}\}_{k=0}^{\dim(\CCX)-1}$, and $\{coA_{k,1}\}_{k=1}^{\dim(\CCX)}$. In fact, $B_{2,3}$ is just a column of zeros for either case. Thus we do not have an injection, which is necessary for a representation. This problem arises because, unlike a cell complex, cells in a combinatorial complex are not constructed by attaching cells of 1-lower rank.

\begin{proposition}
    Let us call a set of CCs isomorphic to a CC $\CCX$ the ``structure class'' of $\CCX$. The structure class of $\CCX$ is determined by the sequence of (co)-incidence matrices $((co)B_{0,k})_{k=1}^{ \dim(\CCX)}$.
\end{proposition}

\begin{proof}
    Given a combinatorial complex (CC) $(S, \CCX, \rk)$, $\CCX$ is simply a subset of  $\mathcal{P}(S)\setminus\{\emptyset\}$ and thus $\forall x \in \CCX, x \in \mathcal{P}(S)$. Then, given an index $i \in \{ n
    \in \mathbb{Z} \mid 1 \leq n \leq |S| \}$ for every element in $S$, we can represent every $x$ as a binary one-hot encoded vector $v_x$ of length $|S|$ such that the $i$'th element is $1$ if the $i$ indexed element in $S$ is contained in $x$, and $0$ otherwise. Thus we have established an injection from $\CCX$ to $\B^{|S|}$, where $\B^{|S|}$ is the set of binary vectors of length $|S|$. Every $v_x$ that represents an element $x \in \CCX$ with rank $\rk(x) = k$ exists exactly once as a column of the $k$'th matrix in the sequence $(B_{0,k})_{k=1}^{ \dim(\CCX)}$, by construction. Thus we have an injection from CC $(S, \CCX, \rk)$ to $(B_{0,k})_{k=1}^{ \dim(\CCX)}$ and thus a representation. Trivially, since $(coB_{0,k})_{k=1}^{ \dim(\CCX)}$ is simply tarnsposes of $(B_{0,k})_{k=1}^{ \dim(\CCX)}$, this representation is valid for $(coB_{0,k})_{k=1}^{ \dim(\CCX)}$ as well.
\end{proof}

\section{Clique Lifting}
\label{lifting}

\begin{tcolorbox}
[width=\linewidth, sharp corners=all, colback=white!95!black]

\begin{definition}[Clique]
\label{def:clique}
A \textbf{clique} in a graph \( G = (V, E) \) is a subset \( C \subseteq V \) of vertices such that every pair of vertices in \( C \) is connected by an edge in \( G \). In other words, the induced subgraph \( G[C] \) is a complete graph.

Formally, \( C \) is a clique if for all \( u, v \in C \) with \( u \neq v \), the edge \( (u, v) \) is in \( E \).
\end{definition}

\end{tcolorbox}

Clique lifting in its standard form takes a graph, finds all the cliques in this graph, and defines a simplicial complex where $0$-simplices are the nodes, $1$-simplices are the edges, $2$-simplices are the triangles (cliques of the original graph of size $3$), $3$-simplices are tetrahedrons (cliques of the original graph of size 4), and so on. We extend this to CCs: We convert this simplicial complex by converting all  $n$-simplices of $n < 2$ to $n$-cells of a combinatorial complex, and $n$-simplices of $n \geq 2$ to $2$-cells of a combinatorial complex.

We first define CC-embedding. CC-embeddings allow us to formalise lifting a lower-order topological structure such as a graph to a higher-order structure such as a CC. We provide one such CC-embedding derived from the clique-lifting method described above. 

\begin{tcolorbox}
[width=\linewidth, sharp corners=all, colback=white!95!black]

\begin{definition}[CC-Embedding]
\label{maps}
 A homomorphism from a CC
 $ (S_1, \mathcal{X}_1, \rk_1)$ to a CC
 $(S_2, \mathcal{X}_2, \rk_2)$,
also called a \textbf{CC-homomorphism},
 is a function $f \colon \mathcal{X}_1 \to \mathcal{X}_2 $ that satisfies the following conditions:
\begin{enumerate}
%\item If $x\in\mathcal{X}_1$, then $f(x)\in\mathcal{X}_2$.
\item If $x,y\in\mathcal{X}_1$ satisfy $x\subseteq y$, then $f(x) \subseteq f(y)$. 

\item If $x\in\mathcal{X}_1$, then $\rk_1(x)\geq \rk_2(f(x)).$\\
%\item For each $ y \in \Ima(f)$
%and for each $x \in  \mathcal{X}_1 $ such that $y=f(x)$, 
%it holds that $\rk_1(x) \geq \rk_2(y)$.

\end{enumerate}
With the following additional conditions:
\begin{enumerate}

\item If $x\in\mathcal{X}_1$, then $\rk_1(x) = \rk_2(f(x)).$ This constrains inequality in conditions of strict equality.

\item $\forall x\in\mathcal{X}_1, f$ is injective.

\end{enumerate}

We call a such CC-homomorphism a \textbf{CC-embedding}.
\end{definition}
\end{tcolorbox}

% Given a CC $ (S, \mathcal{X}, \rk)$, let $\mathcal{X}^k$ denote set of cells of rank $k$.
Consider a graph \( G = (V, E) \), as in Definition~\ref{graph:main}. Let $C(G)$ denote the set of all cliques in $G$. We can view this graph $G$ as a $CC (S_1 = V, \quad \mathcal{X}_1 = V \cup E, \quad\rk_1)$, where $\rk_1(x) = |x|-1$. We define a new CC $(S_2, \mathcal{X}_2, \rk_2)$, where $S_2 = V$, $\mathcal{X}_2 = \mathcal{X}_1 \cup C(G)$, and the rank function defined by:

\begin{equation}
\rk_2(x) =
\begin{cases}
    0 & \text{if } x \in V \\
    1 & \text{if } x \in E \\
    2 & \text{if } x \in C(G).
\end{cases}
\end{equation}

Then we trivially define the CC-embedding ``Clique Lifting'' $f_{CL}: \mathcal{X}_1 \rightarrow \mathcal{X}_2$ by $f_{CL}(x) = x$. Note that $f_{CL}$ is injective but not surjective.

%%%%%%%%%%%%%%%%%%%%%%%%%%%%%%%%%%%%
\chapter{Methodology \& Contribution}

\section{CC Loss Functions}

\label{loss_functions}

In this section, we present variations of what we dub \textbf{Row-wise Permutation Invariant Loss (RWPL)} function used for both training and evaluation. 

Let us qualitatively describe RWPL: For each rank, we take 2 co-incidence matrices representing cells of some rank of 2 different CCs. These matrices have the same number of columns and possibly a different number of rows, since the set $S$ is the same for every CC in the series, hence the number of $0$-cells is as well, while $1$-cells form and dissolve, varying their number in each snapshot. Column indices are $0$-cell indices and row indices are n-cell indices (for our evaluations it will be $1$ and $2$-cells). Each row in these binary matrices represents an n-cell in the CC, where the i'th element in the row is $1$ if the i'th $0$-cell is contained in this n-cell, and 0 otherwise. We pad these matrices to have the same size. This way one of the matrices will have some number of $0$ rows added to match the row number of the other. For each row on one matrix, we compute some similarity metric for every row in the other matrix. Then, using an optimal matching algorithm, for every row in the first matrix, we find a unique pair from the other matrix. We then aggregate the similarity metric of choice over all the optimally matched rows. 

Optimal matching of rows ensures that our loss function is row-wise permutation invariant. This is desirable because, unlike the usual case of adjacency matrices, incidence matrices make cell indexing very challenging. To see this imagine a CC time series where every cell in timestep $0$ is indexed and an incidence matrix is constructed accordingly. If the $0$ indexed cell somehow vanishes (or gets ``divided'', or loses some $0$-cells, or gets modified in a significant way), this cell won't show on the $0$'th row on the incidence matrix and all the other cell indices will have to be adjusted. The author of this paper could not find a suitable way to adjust these indices. If we were to just compute the losses between the i'th row from the first matrix and the i'th row from the other matrix, we would effectively be comparing arbitrary cells. Row-wise permutation invariance ensures that if same or similar cells exist in sampled and target CCs, we compute the loss between them and not arbitrary cells.

To solve this problem, one approach could be to index each row by the ordering of all possible cells given $S$. This is the approach taken in \cite{carrel_combinatorial_2024}, where they have introduced ``Dimension-Constrained Featured Combinatorial Complexes'' so that a maximum number of elements from $S$ each cell can contain was decided in advance and each row of the incidence matrix corresponded to a possible cell. This is computationally extremely expensive, which CCSD suffered from. More formally: Given a set \( S \) with \( |S| = k \) elements, the number of subsets of \( S \) that can contain up to \( N \) elements is given by the following sum of binomial coefficients:

\[
\sum_{n=0}^{N} \binom{k}{n}
\]

Where:
\begin{itemize}
    \item \( k \) is the size of the set \( S \),
    \item \( N \) is the maximum number of elements allowed in the cells,
    \item \( \binom{k}{n} \) represents the binomial coefficient, which counts the number of ways to choose \( n \) elements from \( k \) elements.
\end{itemize}

This means our incidence matrices would need to have $\sum_{n=0}^{N} \binom{k}{n}$ many rows, most of which are $0$ rows, which quickly becomes unfeasible for larger graphs like the ones we would like to model. Therefore, we take incidence matrices that only contain rows corresponding to cells that exist in the CC with no zero rows and compute row-wise permutation invariant losses.

Below we give a more rigorous description of RWPL with varied row losses and optimal matchings.

\subsection{Pairwise Binary Cross-Entropy (BCE) Loss}

Let \( \B = \{0,1\} \) denote the set of binary values. Given two matrices \( \mathbf{A} \in \B^{N_A \times D} \) and \( \mathbf{B} \in \B^{N_B \times D} \), the pairwise BCE loss between each pair of rows \( \mathbf{a}_i \) from \( \mathbf{A} \) and \( \mathbf{b}_j \) from \( \mathbf{B} \) is computed element-wise as:

\[
\text{BCE}(\mathbf{a}_i, \mathbf{b}_j) = - \sum_{d=1}^D \left( b_{j,d} \log a_{i,d} + (1 - b_{j,d}) \log(1 - a_{i,d}) \right) = C_{ij}.
\]

Here, the BCE loss is computed for each corresponding element \( a_{i,d} \) from row \( \mathbf{a}_i \) and \( b_{j,d} \) from row \( \mathbf{b}_j \), where \( d \) indexes the dimension. The total loss for each pair of rows \( \mathbf{a}_i \) and \( \mathbf{b}_j \) is then obtained by summing the element-wise losses across all dimensions.

This computation results in a pairwise distance matrix \( C \in \mathbb{R}^{N_A \times N_B} \), where each element \( C_{ij} \) represents the BCE loss between the \(i\)-th row of \( \mathbf{A} \) and the \(j\)-th row of \( \mathbf{B} \).

\subsection{Pairwise Cosine Loss}

Given two matrices \( \mathbf{A} \in \B^{N_A \times D} \) and \( \mathbf{B} \in \B^{N_B \times D} \), the cosine similarity between rows \( \mathbf{a}_i \) from \( \mathbf{A} \) and \( \mathbf{b}_j \) from \( \mathbf{B} \) is computed as:

\[
S_{ij} = \frac{\mathbf{a}_i \cdot \mathbf{b}_j}{\|\mathbf{a}_i\|_2 \|\mathbf{b}_j\|_2}
\]

where \( S \in \B^{N_A \times N_B} \) is the similarity matrix.

The cosine distance matrix \( C \) is derived from the similarity matrix \( S \):

\[
C_{ij} = 1 - S_{ij}.
\]

This yields the distance matrix \( C \in \B^{N_A \times N_B} \), where each element \( C_{ij} \) measures the dissimilarity between rows \( \mathbf{a}_i \) and \( \mathbf{b}_j \).

\subsection{Optimal Matching}

We use Hungarian \cite{hungarian} and Sinkhorn \cite{Sinkhorn1967ConcerningNM} Algorithms. We provide both of these algorithms in Appendix~\ref{appendix_c}. Hungarian Algorithm is guaranteed to provide the best possible matching, however, we can't use it as a component of the loss function in optimization as it involves discrete operations (like row and column reductions, matching, and assignments) that do not have a gradient. Sinkhorn Algorithm is differentiable and faster but is not guaranteed to provide the best matching. We use 50 iterations of Sinkhorn algorithm in our training and evaluation. In evaluation, we see the output of Sinkhorn to be fairly close to Hungarian. 

The objective of both of these algorithms is to find a permutation \( \sigma: \{1, \dots, N_A\} \rightarrow \{1, \dots, N_B\} \) that aims to minimise the total cost:

\[
\text{RWPL} = \text{Total Cost} = \min_{\sigma} \sum_{i=1}^{N_A} C_{i \sigma(i)}.
\]

Thus we end up with four Row-wise Permutation Invariant Loss functions: \textbf{Hungarian BCE}, \textbf{Sinkhorn BCE}, \textbf{Hungarian Cosine}, and \textbf{Sinkhorn Cosine}.

\section{Model}

 The goal of our model is to learn an approximation for $p(\cdot \mid  (CC_{t-1}))$ from which we can sample $\hat{CC}_{t}$. This is the same approach taken in DAMNETS. However, in DAMNETS, as they work in adjacency matrices with the same number of nodes, it sufficed for them to work with delta matrices which encoded whether an edge remained, was added, or taken away. As also mentioned in Section~\ref{loss_functions}, in incidence matrices we can't work with delta matrices because in an incidence matrix, if the $0$ indexed cell somehow vanishes (getting ``divided'', losing some $0$-cells, or getting modified in a significant way), this cell won't show on the $0$'th row on the incidence matrix and all the other cell indices will have to be adjusted. 
If we were to only compute the losses between the i'th row from the first matrix and the i'th row from the other matrix, we'd effectively be comparing arbitrary cells. Therefore, we directly work on the incidence matrix representation outlined in Section~\ref{representation}.

Still, similar to DAMNETS, our approach is an encoder-decoder framework. We summarise the previous CC by computing cell cochains for every cell in the CC by computing an embedding per every row in the incidence matrices using 2 different encoders explained in the next section. Then we use a decoder similar to but simpler than DAMNET's modified version of BiGG, which we implemented from scratch.

\subsection{Encoder}

\subsubsection{HOAN}
To create the cochains, we use a modified version of Higher-order Attention Network (HOAN), from \cite{hajij_higher-order_2022}, without the mesh classification readout originally implemented, since we are not classifying anything. We modify the readily implemented HOAN using HMC (Higher-order Mesh Classifier) layers from TopoModelX available in \href{https://github.com/pyt-team/TopoModelX?tab=readme-ov-file}{this link}.

First, for every Graph in the dataset, we convert the Graph to a CC using clique lifting and extract $A_{0,1}, A_{1,2}, coA_{1,2}, B_{0,1}, B_{1,2}$. If node features are defined in the data, we construct $ \textbf{H}_0$ using these features. If node features are not defined, we let $\textbf{H}_0 = \textbf{I}_{|\CCX^0|}$. Then we set $\textbf{H}_1=\textbf{1}_{|\CCX^1|}, \textbf{H}_2=\textbf{1}_{|\CCX^2|}$, where $\textbf{1}_{N}$ refers to a vector of $1$s with length $N$ and $\textbf{I}_{N}$ refers to a square $N \times N$ identity matrix.. Let $h_y^{t,(r)}$ denote the cochain vector of cell $y$ of rank $r$ after encoding at iteration $t$.
We encode $\textbf{H}_1, \textbf{H}_2$ using 2 consecutive HMC layers. Each HMC has two levels. The message passing scheme equations for the 2 levels are denoted below. These equations were taken from the HMC tutorial of TopoModelX, adapted to our notation:

\subsubsection{First Level}

\begin{align}
  m^{0\rightarrow 0}_{y\rightarrow x} &= \left((A_{0,1})_{xy} \cdot \text{att}_{xy}^{0\rightarrow 0}\right) h_y^{t,(0)} W^t_{0\rightarrow 0}, &\text{att. push-forward from 0 to 0-cell} \\
  m^{0\rightarrow 1}_{y\rightarrow x} &= \left((B_{0,1}^T)_{xy} \cdot \text{att}_{xy}^{0\rightarrow 1}\right) h_y^{t,(0)} W^t_{0\rightarrow 1}, &\text{att. push-forward from 0 to 1-cell} \\
  m^{1\rightarrow 0}_{y\rightarrow x} &= \left((B_{0,1})_{xy} \cdot \text{att}_{xy}^{1\rightarrow 0}\right) h_y^{t,(1)} W^t_{1\rightarrow 0}, &\text{att. push-forward from 1 to 0-cell} \\
  m^{1\rightarrow 2}_{y\rightarrow x} &= \left((B_{1,2}^T)_{xy} \cdot \text{att}_{xy}^{1\rightarrow 2}\right) h_y^{t,(1)} W^t_{1\rightarrow 2}, &\text{att. push-forward from 1 to 2-cell} \\
  m^{2\rightarrow 1}_{y\rightarrow x} &= \left((B_{1,2})_{xy} \cdot \text{att}_{xy}^{2\rightarrow 1}\right) h_y^{t,(2)} W^t_{2\rightarrow 1}, &\text{att. push-forward from 2 to 1-cell}
\end{align}

\begin{align}
  m^{0\rightarrow 0}_{x} &= \phi_{intra}\left(\bigoplus_{y \in A_{0,1}(x)} m^{0\rightarrow 0}_{y\rightarrow x}\right), &\text{intra-n. aggr. in 0-cells} \\
  m^{0\rightarrow 1}_{x} &= \phi_{intra}\left(\bigoplus_{y \in B_{0,1}^T(x)} m^{0\rightarrow 1}_{y\rightarrow x}\right), &\text{intra-n. aggr. from 0-cell to 1-cell} \\
  m^{1\rightarrow 0}_{x} &= \phi_{intra}\left(\bigoplus_{y \in B_{0,1}(x)} m^{1\rightarrow 0}_{y\rightarrow x}\right), &\text{intra-n. aggr. from 1-cell to 0-cell} \\
  m^{1\rightarrow 2}_{x} &= \phi_{intra}\left(\bigoplus_{y \in B_{1,2}^T(x)} m^{1\rightarrow 2}_{y\rightarrow x}\right), &\text{intra-n. aggr. from 1-cell to 2-cell} \\
  m^{2\rightarrow 1}_{x} &= \phi_{intra}\left(\bigoplus_{y \in B_{1,2}(x)} m^{2\rightarrow 1}_{y\rightarrow x}\right), &\text{intra-n. aggr. from 2-cell to 1-cell}
\end{align}

\begin{align}
  m_x^{(0)} &= \phi_{inter}\left(m^{0\rightarrow 0}_{x} + m^{1\rightarrow 0}_{x}\right), &\text{inter-n. aggr. in 0-cells} \\
  m_x^{(1)} &= \phi_{inter}\left(m^{0\rightarrow 1}_{x} + m^{2\rightarrow 1}_{x}\right), &\text{inter-n. aggr. in 1-cells} \\
  m_x^{(2)} &= \phi_{inter}\left(m^{1\rightarrow 2}_{x}\right), &\text{inter-n. aggr. in 2-cells}
\end{align}

\begin{align}
  i_x^{t,(0)} &= m_x^{(0)}, &\text{intermediate feature vector for 0-cells} \\
  i_x^{t,(1)} &= m_x^{(1)}, &\text{intermediate feature vector for 1-cells} \\
  i_x^{t,(2)} &= m_x^{(2)}, &\text{intermediate feature vector for 2-cells}
\end{align}

where \(i_x^{t,(\cdot)}\) represents intermediate feature vectors.

\subsubsection{Second Level}

\begin{align}
  m^{0\rightarrow 0}_{y\rightarrow x} &= \left((A_{0,1})_{xy} \cdot \text{att}_{xy}^{0\rightarrow 0}\right) i_y^{t,(0)} W^t_{0\rightarrow 0}, &\text{att. push-forward in 0-cells} \\
  m^{1\rightarrow 1}_{y\rightarrow x} &= \left((A_{1,2})_{xy} \cdot \text{att}_{xy}^{1\rightarrow 1}\right) i_y^{t,(1)} W^t_{1\rightarrow 1}, &\text{att. push-forward in 1-cells} \\
  m^{2\rightarrow 2}_{y\rightarrow x} &= \left((coA_{1,2})_{xy} \cdot \text{att}_{xy}^{2\rightarrow 2}\right) i_y^{t,(2)} W^t_{2\rightarrow 2}, &\text{att. push-forward in 2-cells} \\
  m^{0\rightarrow 1}_{y\rightarrow x} &= \left((B_{0,1}^T)_{xy} \cdot \text{att}_{xy}^{0\rightarrow 1}\right) i_y^{t,(0)} W^t_{0\rightarrow 1}, &\text{att. push-forward from 0 to 1-cell} \\
  m^{1\rightarrow 2}_{y\rightarrow x} &= \left((B_{1,2}^T)_{xy} \cdot \text{att}_{xy}^{1\rightarrow 2}\right) i_y^{t,(1)} W^t_{1\rightarrow 2}, &\text{att. push-forward from 1 to 2-cell}
\end{align}

\begin{align}
  m^{0\rightarrow 0}_{x} &= \phi_{intra}\left(\bigoplus_{y \in A_{0,1}(x)} m^{0\rightarrow 0}_{y\rightarrow x}\right), &\text{intra-n. aggr. in 0-cells} \\
  m^{1\rightarrow 1}_{x} &= \phi_{intra}\left(\bigoplus_{y \in A_{1,2}(x)} m^{1\rightarrow 1}_{y\rightarrow x}\right), &\text{intra-n. aggr. in 1-cells} \\
  m^{2\rightarrow 2}_{x} &= \phi_{intra}\left(\bigoplus_{y \in coA_{1,2}(x)} m^{2\rightarrow 2}_{y\rightarrow x}\right), &\text{intra-n. aggr. in 2-cells} \\
  m^{0\rightarrow 1}_{x} &= \phi_{intra}\left(\bigoplus_{y \in B_{0,1}^T(x)} m^{0\rightarrow 1}_{y\rightarrow x}\right), &\text{intra-n. aggr. from 0-cell to 1-cell} \\
  m^{1\rightarrow 2}_{x} &= \phi_{intra}\left(\bigoplus_{y \in B_{1,2}^T(x)} m^{1\rightarrow 2}_{y\rightarrow x}\right), &\text{intra-n. aggr. from 1-cell to 2-cell}
\end{align}

\begin{align}
  m_x^{(0)} &= \phi_{inter}\left(m^{0\rightarrow 0}_{x} + m^{1\rightarrow 0}_{x}\right), &\text{inter-n. aggr. in 0-cells} \\
  m_x^{(1)} &= \phi_{inter}\left(m^{1\rightarrow 1}_{x} + m^{0\rightarrow 1}_{x}\right), &\text{inter-n. aggr. in 1-cells} \\
  m_x^{(2)} &= \phi_{inter}\left(m^{1\rightarrow 2}_{x} + m^{2\rightarrow 2}_{x}\right), &\text{inter-n. aggr. in 2-cells}
\end{align}

\begin{align}
  h_x^{t+1,(0)} &= m_x^{(0)}, &\text{final feature vector for 0-cells} \\
  h_x^{t+1,(1)} &= m_x^{(1)}, &\text{final feature vector for 1-cells} \\
  h_x^{t+1,(2)} &= m_x^{(2)}, &\text{final feature vector for 2-cells}
\end{align}

Inter-neighbourhood and intra-neighbourhood aggregations from Definition~\ref{homp-definition} are abbreviated as ``inter-n. aggr." and ``intra-n. aggr." Attention is abbreviated as ``att."
Note that $\bigoplus = \sum$ for all the above equations. $\bigotimes$ can be simply replaced with $(+)$, however we didn't for the sake of sticking to the original notation. In both message passing levels, \(\phi_{intra}\) and \(\phi_{inter}\) represent common activation functions for within and between neighbourhood aggregations, respectively. Also, \(W\) and \(\text{att}\) represent learnable weights and attention matrices, respectively, that are different at each level. We use LeakyReLU activation function for the attention matrix, as in \cite{hajij_higher-order_2022}. In our experiments we chose $h^{2,(1)}$, and $h^{2,(1)}$ to have length 256, same as in DAMNETS. Note that, unlike encoding nodes in DAMNETS, we encode 1-cells and 2-cells.

\subsection{Decoder}

After the encoding stage we have the feature matrices $\mathbf{H_{1,enc}}, \mathbf{H_{2,enc}}$  for the cochain maps $\mathcal{C}^1, \mathcal{C}^2$ which denote the encoded features of $1$ and $2$-cells respectively. We follow the same decoding steps for each matrix. We explain the process for $\mathbf{H_{enc}^{t}}$, which may refer to either of these matrices at timestep $t$. Let us denote $coB_{0,r}^t$ at timestep $t$ by $coB^t$. Given $coB^{t}$ we would like to generate a predicted $coB^{t+1}$, which we denote by $\hat{coB}^{t+1}$. Let $\CCX^{r, index}$ denote the index set of the cells in $\CCX^{r}$ Given a cell ordering, we start with the first cell in $\CCX^{r, index}$. We would like to generate $\hat{coB}$ row by row. Conditioning every row on the previous rows

\[
p(\hat{coB}) = p(\{\hat{coB}_{c}\}_{c\in \CCX^{r, index}}) = \prod_{c \in \CCX^{r, index}} p(\hat{coB}_{c} \mid\{\hat{coB}_{c_{prev}}: c_{prev} < c \}).
\]

Where $\hat{coB}_{c}$ denotes the $c$'th row in the matrix $\hat{coB}$. So every row on the co-incidence matrix depends on all the rows that come before. The algorithm for generation of the full matrix $\hat{coB}_{c}$ is included as pseudo-code in Algorithm~\ref{algo1}. The algorithm of row generation follows a tree structure, which we include as pseudo-code as Algorithm~\ref{algo2}.

We now explain Algorithm~\ref{algo1}. Here, $g$ is a vector we use to encode the decisions of \texttt{Traverse\_Tree} for the previously sampled rows of  $\hat{coB}^{t+1}$. $k$ is a variable we use to track which row we are at. $N_{\text{new cell}}$ is a hyper-parameter we set in advance. It denotes how many times we iterate over the rows of $\textbf{H}_{enc}^{t}$. The rationale for adding $N_{\text{new cell}}$ is to allow each cell to ``divide" into 
a maximum number of $N_{\text{new cell}}$ cells. This allows us to grow the number of cells rather than always sampling a fixed number of cells. One might think this will always lead to $N_{\text{new cell}}$ times many more cells. However, we allow \texttt{Traverse\_Tree} to sample zero rows, and we discard them before we return  $\hat{coB}^{t+1}$ to the loss function. By carefully tuning the hyper-parameter $P_{min}$, we managed to sample a similar number of cells to  $\hat{coB}^{t}$. This formulation allows the network to both grow if few $0$-cells are sampled and diminish if many $0$-cells are sampled. At line 7, $h_{root} \gets \text{MLP}_{cat}(h, g)$ is an MLP which processes the concatenation of $h$ and $g$. This is the same idea employed in DAMNETS and allows us to combine the cell encoding $h$ from the previous CC with the row-wise auto-regressive term $g$. In our experiments, the variable $min_{nonzero}$ is a hyper-parameter we set to $1$ so that if we sample a cell which contains only one element from $S$, we resample as this violates conditions of a CC.

\begin{algorithm}[htbp]
\caption{\texttt{Incidence\_Matrix\_Sampler}}
\label{algo1}
\begin{algorithmic}[1]
\STATE \textbf{Input:} $\textbf{H}_{enc}^{t}$
\STATE $g \gets \mathbf{0}$ 
\STATE $k \gets 0$
\STATE $\hat{coB}^{t+1} \gets \mathbf{0}_n$ \COMMENT{Initialize sample matrix}

\FOR{$j \gets 1$ \TO $N_{\text{new cell}}$}
    \FOR{each row $h$ in $\textbf{H}_{enc}^{t}$}
        \STATE $h_{root} \gets \text{MLP}_{cat}(h, g)$
        \STATE $valid\_sample \gets \text{False}$
        \WHILE{not $valid\_sample$}
            \STATE $\hat{coB}^{t+1}_k, g_{new} \gets \texttt{Sample\_Row}(h_{root})$
            \STATE $n_{nonzeros} \gets \texttt{count\_nonzeros}(\hat{coB}^{t+1}_k)$
            \IF{$n_{nonzeros} = 0$ \textbf{or} $(min_{nonzero} \leq n_{nonzeros}$}
                \STATE $valid\_sample \gets \text{True}$
            \ENDIF
        \ENDWHILE
        \STATE $k \gets k + 1$
        \STATE $g \gets \text{TFEncoder}(g_{new} ; g_{1:k-1})$
    \ENDFOR
\ENDFOR

\STATE \textbf{Return} $\hat{coB}^{t+1}$
\end{algorithmic}
\end{algorithm}

Let us follow the traverse of Algorithm~\ref{algo2} with an example. 

\begin{example}
    \label{tree}
Let $|S|=8$. Then we are to sample a binary vector of length $8$. On every tree node, the pink squares represent the indexes of the binary vector each tree node influences. The numbers above the tree nodes represent at which steps those nodes are visited. In the traverse, we always attempt to descend to the left child, and then right. Below we list all the decisions made in the tree traverse:

\begin{enumerate}
    \item We start at the root node $[1,8]$. Via a decision making mechanism explained in more detail in the paragraph after this list, we decide to descend to the left child $[1,4]$.
    \item  Here we again decide to descend to the left child $[1,2]$.
    \item We decide to not descend to left child $[1,1]$.
    \item We decide to not descend to left child $[1,2]$ either. So we are moving back to $[1,4]$
    \item Since we had already descended left, we make a decision about the right child. We decide to descend to the right child $[3,4]$.
    \item We decide to descend to left child leaf node $[3,3]$.
    \item Here, we need to make a decision as to whether we will sample the leaf. We choose to sample. We go back up to $[3,4]$.
    \item Since we had already descended left, we make a decision about the right child. We decide not to descend to the right child $[4,4]$. Since all traverse possibilities are exhausted on the left child of the root node, we go back to the root node $[1,8]$ to decide for its right child.
    \item We decide to descend to the right child $[5,8]$.
    \item We decide to descend to the left child $[5,6]$.
    \item We decide to descend to the left child $[5,5]$.
    \item We decide to sample. Now we go back up to $[5,6]$.
    \item We decide to descend to the right child $[5,5]$.
    \item We decide not to sample. Since we have exhausted all the possible traversals in $[5,6]$, we go back up to $[5,8]$.
    \item We decide to descend to the right child $[7,8]$.
    \item We decide not to descend to the left child $[7,7]$.
    \item We decide to descend to the right child $[8,8]$.
    \item We decide to sample. We have completed the tree traverse.
\end{enumerate}

Thus our row sample is 
\[
    \hat{coB}_k = \begin{pmatrix} 0 & 0 & 1 & 0 & 1 & 0 & 0 & 1 \end{pmatrix}.
\]

\end{example}

\begin{figure}[htbp]
\centering
\includegraphics[width = 0.8\hsize]{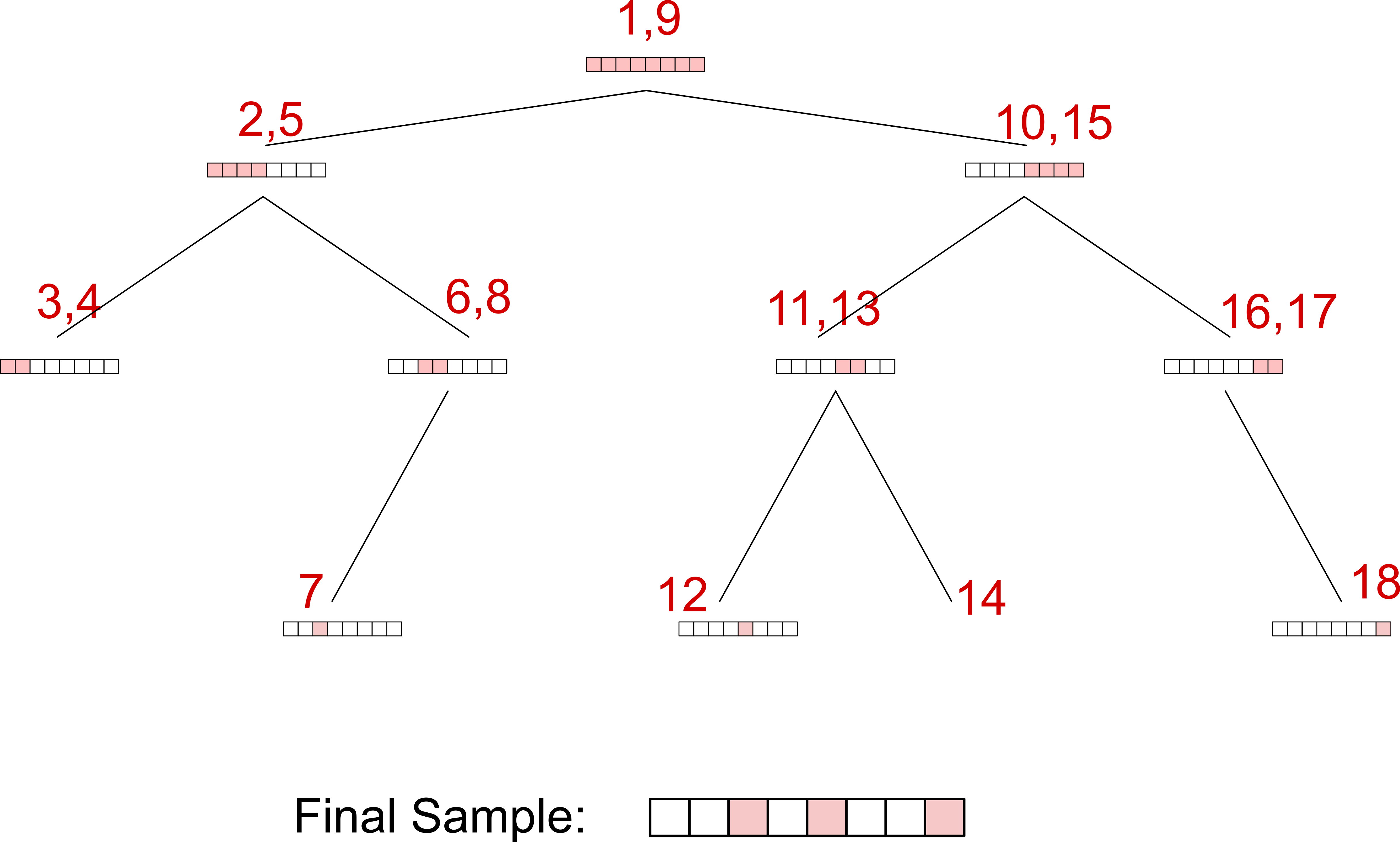}
\caption{A visual representation of an example traverse of Algorithm~\ref{algo2}.}
\label{fig:traverse}
\end{figure}

We now explain Algorithm~\ref{algo2} in detail. $\mathcal{T}_{node}$ denotes which node of the tree we are at. In this implementation $\mathcal{T}_{node}$ is represented by a tuple $(a,b)$ which denotes the range of indices in  $\hat{coB}^{t+1}_k$ that the node is responsible for. $n_{nonzeros}$ denote the number of nonzero elements in $\hat{coB}^{t+1}_k$. $n_{max}$ is a hyper-parameter we set in advance which allows us to sample cells that contain up to a given maximum number of elements from $S$. This is useful because if we know how many elements we know a cell will have (for instance, an edge can only connect 2 nodes), we can include this in the model. In our experiments, we set $n_{max} = 2$ for $1$-cells and $n_{max} = 15$ for $2$-cells. ``If'' statement in line 4 allows us to return $\hat{coB}^{t+1}_k$ immediately if this number is reached. Once a leaf is reached, $\text{MLP}_{leaf}(h)$ in line 13 provides the probability that we will sample the index corresponding to that leaf given the tree state encoded up until that point by $h$. DAMNETS has a similar mechanism. If we are not at a leaf, we traverse through the tree by deciding whether there is a leaf under the left child of the current node. This is decided by $\text{MLP}_{left}$ in line 11, which takes the encoded state of the tree and outputs the parameter for relaxed Bernoulli. We could have omitted relaxed Bernoulli, but we have included it because we wanted some diversity when we first started training the model. Otherwise, when we first began to train, we observed very similar samples. The traverse to the right tree node is similar. We use the LSTM\cite{lstm} cell in lines 15 and 20 to encode the history of the tree traverse into hidden state vector $h$ and internal cell state $c$. Tree\_LSTM \cite{treelstm} in line 23 allows us to encode the decisions on the lower levels of the tree.

\begin{algorithm}[htbp]
\caption{\textnormal{\texttt{Traverse\_Tree}}}
\label{algo2}
\begin{algorithmic}[1]
\STATE \textbf{Input:} $h$, $\mathcal{T}_{node} \gets \mathcal{T}_{root}$, $c \gets \mathbf{0}_{size(h)}$, $\hat{coB}^{t+1}_k \gets \mathbf{0}_n$ (Default values are assigned by $\gets$ if not given.)
\STATE \textbf{Output:} $\hat{coB}^{t+1}_k$ (updated), $h$ (updated), $c$ (updated)
\\
\STATE Initialize $h_{left}$, $c_{left}$, $h_{right}$, $c_{right}$ as None (if not yet assigned a value.)
\STATE $n_{nonzeros} \gets \texttt{count\_nonzeros}(\hat{coB}^{t+1}_k)$
\IF{$n_{nonzeros} = n_{max}$}
    \STATE \textbf{Return} $\hat{coB}^{t+1}_k$, $h$, $c$
\ENDIF
\\
\IF{$\mathcal{T}_{node}$ is a leaf}
    \STATE $i \gets \text{index corresponding to leaf } \mathcal{T}_{node}$
    \STATE $\hat{coB}^{t+1}_k[i] \gets \text{Bernoulli}(\text{MLP}_{leaf}(h))$
    \STATE \textbf{Return} $\hat{coB}^{t+1}_k$, $h$, $c$
\ENDIF
\\
\STATE $P_{left} \gets \text{Relaxed\_Bernoulli}(\text{MLP}_{left}(h))$
\IF{$P_{left} > P_{min}$}
    \STATE $h_{left}, c_{left} \gets \text{LSTM}(h, c)$
    \STATE $\hat{coB}^{t+1}_k, h_{left}, c_{left} \gets \texttt{Traverse\_Tree}(h_{left}, \texttt{Left\_Child}(\mathcal{T}_{node}), c_{left}, \hat{coB}^{t+1}_k)$
\ENDIF
\\
\STATE $P_{right} \gets \text{Relaxed\_Bernoulli}(\text{MLP}_{right}(h))$
\IF{$P_{right} > P_{min}$}
    \STATE $h_{right}, c_{right} \gets \text{LSTM}(h, c)$
    \STATE $\hat{coB}^{t+1}_k, h_{right}, c_{right} \gets \texttt{Traverse\_Tree}(h_{right}, \texttt{Right\_Child}(\mathcal{T}_{node}), c_{right}, \hat{coB}^{t+1}_k)$
\ENDIF
\\

\STATE $h_{combined}, c_{combined} \gets \text{Tree\_LSTM}(h_{left}, c_{left}, h_{right}, c_{right})$

% \IF{$h_{left} \neq \text{None}$ and $h_{right} \neq \text{None}$}
%     \STATE $h_{combined}, c_{combined} \gets \texttt{Tree\_Lstm}(h_{left}, c_{left}, h_{right}, c_{right})$
% \ELSIF{$h_{left} \neq \text{None}$}
%     \STATE $h_{combined}, c_{combined} \gets h_{left}, c_{left}$
% \ELSIF{$h_{right} \neq \text{None}$}
%     \STATE $h_{combined}, c_{combined} \gets h_{right}, c_{right}$
% \ELSE
%     \STATE $h_{combined}, c_{combined} \gets h, c$
% \ENDIF
\STATE \textbf{Return} $\hat{coB}^{t+1}_k$, $h_{combined}$, $c_{combined}$
\end{algorithmic}
\end{algorithm}

%%%%%%%%%%%%%%%%%%%%%%%%%%%%%%%%%%%%
\chapter{Experiments}

\label{sec:experiments}
\section{Problem Statement}

\textbf{Graphs:} Given a timestamp \(t\), we represent each graph in the data as a tuple of its adjacency and node feature matrices: $G_t = (\mathbf{X}_t, \mathbf{A}_t)$, where  $\mathbf{X}_t \in \R^{N \times F} \text{ and } \mathbf{A}_t \in \{0,1\}^{N \times N}$. Here, $N$ is the total number of nodes, and $F$ is the number of node features. We then represent a graph time series as a sequence of graphs over time: \((G_t)_{t=0}^T = (G_1, G_2, \dots, G_T) \), where $T$ is the last timestamp. For our synthetically generated datasets, we have more than one such sequence of graphs. Therefore, we represent the full dataset as a set of graph sequences: \[ \{(G_t^1)_{t=0}^T, (G_t^2)_{t=0}^T, \dots, (G_t^{L-1})_{t=0}^T, (G_t^L)_{t=0}^T. \}\]
Given an unseen graph \( G_t\) from such a dataset, our task is to predict \(G_{t+1}\). \\

\textbf{Combinatorial Complexes:} Just as above, given a timestamp \(t\), we represent each CC in the data as a tuple of its co-incidence and node feature matrices: $CC_t = (\mathbf{X}_t, \mathbf{\mathbf{coB}}_{01t}, \mathbf{\mathbf{coB}}_{02t})$, where  $\mathbf{X}_t \in \R^{N \times F} \text{, } \mathbf{\mathbf{coB}}_{01t} \in \{0,1\}^{E \times N}, \text{ and } \mathbf{\mathbf{coB}}_{02t} \in \{0,1\}^{M \times N}$. Here, $N$ is the total number of 0-cells, $E$ is the total number of 1-cells, $M$ is the total number of 2-cells, and $F$ is the number of node features. We then represent a CC time series as a sequence of CCs over time: \((CC_t)_{t=0}^T = (CC_1, CC_2, \dots, CC_T) \), where $T$ is the last timestamp. For our synthetically generated datasets, we have more than one such sequence of CCs. Therefore, we represent the full dataset as a set of CC sequences: \[ \{(CC_t^1)_{t=0}^T, (CC_t^2)_{t=0}^T, \dots, (CC_t^{L-1})_{t=0}^T, (CC_t^L)_{t=0}^T. \}\]
Given an unseen graph \( CC_t\) from such a dataset, our task is to predict \(CC_{t+1}\). Note that our datasets are originally graphs. We lift these graph datasets into CC datasets using the lifting method explained in Section~\ref{lifting}. We resort to this method as there exist no temporal datasets in CC form.\\

\section{Metrics}

For autoregressive graph generation tasks, we have used metrics based on the following graph statistics: degree, degree centrality, local clustering coefficient, closeness centrality, eigenvalue centrality, average clustering coefficient, transitivity, spectrum of the graph laplacian, temporal correlation, temporal closeness, temporal clustering coefficient, all of which are explained in Appendix~\ref{appendix_a} and ~\ref{appendix_b}.

Evaluating generated data is a challenging task. Within the literature, for graphs, the dominant approach is to compute some statistics (like the above) about the sampled and target graphs, compute a histogram of these properties over the nodes or vertices in the graph, use a kernel (often a Gaussian) to compute the Maximum Mean Discrepancy (MMD) \cite{gretton_kernel_2012} which can be thought of a distance function to compare distributions. This is the approach taken in DAMNETS, AGE, and BiGG. However, we take an alternative approach introduced in \cite{souid_temporal_2024} as it is better suited to the time-series nature of our data. This approach is as follows: Firstly, the graph statistics mentioned above are computed for generated sequences. For global graph statistics like average clustering coefficient and transitivity, we take the scalar statistic as is. For local statistics (such as ones computed per node), we generate a histogram for the given statistic over the nodes of the graph, ie, a distribution for the given statistic and a given graph. We then take the predefined quantiles of this distribution. In \cite{souid_temporal_2024} and this paper, we take the first and third quartiles as the quantiles. Finally, we compute the multidimensional Dynamic Time Warping distance (DTW) \cite{shokoohi-yekta_generalizing_2017} between the generated graph sequence statistics and target sequence graph statistics. For all these metrics we have used the implementation in \cite{souid_temporal_2024}. Note that the authors' code is not public.

For evaluating CC data, we use the loss functions described in Section~\ref{loss_functions}. We abbreviate these loss functions as follows: SC - \textit{Sinkhorn Cosine}, HC - \textit{Hungarian Cosine}, SBCE - \textit{Sinkhorn BCE}, HBCE - \textit{Hungarian BCE Loss}. The instances where SBCE blew up are represented by (-) entries in the tables. 

\section{Random Model}
In addition to DAMCC, DAMNETS, and AGE, we implement a ``random model'' as follows: at each timestep, we take target co-incidence matrices $\mathbf{\mathbf{coB}}_{01}^{\text{target}}$ of size $n_{1} \times n_{0}$, and $\mathbf{\mathbf{coB}}_{02}^{\text{target}}$ of size $n_{2} \times n_{0}$, where $n_{0}$, $n_{1}$, $n_{2}$ denote the number of 0-, 1-, and 2-cells, respectively. To generate random CCs, we generate two random matrices of size $n_{1} \times n_{0}$ and $n_{2} \times n_{0}$. We constrain the number of 1's in each row to be in $N_{1s} = \{0\} \cup [\text{min}, \text{max}]$. For $\mathbf{\mathbf{coB}}_{01}^{\text{random}}$ we set $\text{min} = \text{max} = 2$; for $\mathbf{\mathbf{coB}}_{02}^{\text{random}}$ we set $\text{min} = 3$, $\text{max} = 15$, just as in DAMCC. We uniformly randomly sample an integer $n \in N_{1s}$, and for every row of a random incidence matrix, we again uniformly randomly place $n$ 1's into these rows. Note that these random matrices have the slight advantage of having access to the dimensions of the target matrices, whereas other models do not.

\section{Implementation}

 For experiments, we use Adam\cite{kingma2017adammethodstochasticoptimization} optimizer with Sinkhorn Cosine Loss. During training, we train all models using automated parameter optimization as follows: For AGE and DAMNETS, we decay the learning rate by a factor of 0.1 after 120 epochs of no improvement, keep training until there is no validation loss improvement for 250 epochs, and take the model with best validation loss. AGE and DAMNETS are both implemented in PyTorch but DAMNETS's BiGG component is implemented in C++. Both models are extremely efficient compared to DAMCC. Training Both of these models on all datasets takes less than half an hour in total on all the datasets. For DAMCC, we decay the learning rate by a factor of 0.1 after 10 epochs of no improvement, and keep training until there is no validation loss improvement for 20 epochs. The reason why we chose such lower numbers is that our model's tree search algorithm is quite inefficient and we couldn't implement batching due to the varying size of the incidence matrices. One approach could've been padding to bring the inputs and outputs to the same size, which is a reasonable approach that we can implement in the future. Therefore, our model training is extremely slow and takes around 24 hours to converge. See more on these shortcomings in the discussion. 

We performed all experiments on a brev.dev \footnote{https://www.brev.dev/} remote server running Intel(R) Xeon(R) Platinum 8259CL CPU @ 2.50GHz with 4 cores
 and 1 Nvidia Tesla T4 GPU with 16 GB GPU memory, and 15 GB RAM with operating system Ubuntu 20.04.6 LTS. Our implementation is available on \href{https://github.com/Ata-Tuna/DAMCC_1.2.git}{\textbf{this link}}.

\section{Datasets and Evaluation Results}

\subsection{England Covid}

This dataset is a recorded history of COVID-19 cases in England NUTS3 \cite{noauthor_categorynuts_2013} regions, from 3 March to 12 May. The original paper can be found here: \cite{panagopoulos2021transfergraphneuralnetworks}. In our experiments, we have imported this dataset via PyTorch Geometric Temporal as it was readily available there \cite{rozemberczki_pytorch_2021}. The dataset comprises 53 directed graphs, ordered temporally. Every node denotes a region. Edges of the graphs are originally weighted and denote the number of people who moved from one region to another at time $t$. However, we do not use these weights and convert the graphs to undirected graphs and remove self loops. We use node features.  In detail, we let $\mathbf{x}_u^{(t)} = (c_u^{(t-d)}, \ldots, c_u^{(t)})^\top \in \mathds{R}^d$ be a vector of node attributes, which contains the number of cases for each one of the past $d = 8$ days in region $u$. The authors stated that they have used ``the cases of multiple days instead of just the day before the prediction because case reporting is highly irregular between days, especially in decentralized regions.'' PyTorch Geometric documentation informs that the dataset is segmented in days, yet how many days is somewhat unclear as there are 70 days in the given time window but the dataset has 53 graphs. This dataset was used to evaluate DAMNETS and AGE in \cite{souid_temporal_2024}. However, perhaps due to it only having 53 timestamped graphs, they have trained and tested the models on the full dataset. In our experiments we have divided the dataset into training, validation, and testing sets, with lengths 37, 8, and 8, respectively, keeping the original order of the dataset.

We report graph statistic metrics for this dataset in Table~\ref{encovid-table}, and CC evaluation results in Table~\ref{enc_cc1}, \ref{enc_cc2}.

\begin{table}[ht]
\centering
\caption{Comparison of Models on Various Metrics for England Covid dataset. Lower is better.}
\label{encovid-table}
\begin{tabular}{lcccc}
\toprule
 & \multicolumn{4}{c}{\textit{Models}} \\
\cmidrule{2-5}
{Metrics} & AGE & DAMCC & DAMNETS & RANDOM \\
\midrule
\textit{spectral} & \underline{\textbf{0.001}} & \textbf{0.003} & \underline{\textbf{0.001}} & 0.006\\
\textit{degree} & \underline{\textbf{0.001}} & \textbf{0.003} & \underline{\textbf{0.001}} & 0.007\\
\textit{deg. cent.} & \underline{\textbf{0.000}} & \underline{\textbf{0.000}} & \underline{\textbf{0.000}} & \underline{\textbf{0.000}}\\
\textit{local clust. coeff.} & \underline{\textbf{0.001}} & 0.012 & \textbf{0.002} & 0.012\\
\textit{close. cent.} & \textbf{0.001} & \underline{\textbf{0.000}} & \underline{\textbf{0.000}} & 0.002\\
\textit{eigen. cent.} & \underline{0.005} & \textbf{0.004} & \underline{\textbf{0.003}} & 0.017\\
\textit{ave. clust. coeff.} & 0.381 & \textbf{1.144} & \underline{0.201} & \underline{\textbf{0.138}}\\
\textit{transitivity} & \textbf{0.192} & 0.672 & \underline{\textbf{0.115}} & \underline{0.318}\\
\textit{temp. corr.} & \textbf{0.120} & 0.661 & \underline{\textbf{0.062}} & \underline{0.429}\\
\textit{temp. close. diff.} & \textbf{13.814} & \underline{16.715} & \underline{\textbf{5.097}} & 49.267\\
\textit{temp. clust. coeff. diff.} & \textbf{0.202} & 0.453 & \underline{\textbf{0.113}} & \underline{0.237}\\
\bottomrule
\end{tabular}
\end{table}

\begin{table}[ht]
\centering
\caption{Comparison of Different Losses for $\mathbf{coB}_{0,1}$ across models for England Covid dataset. Lower values are better.}
\label{enc_cc1}
\begin{tabular}{lcccc}
\toprule
 & \textit{SC} & \textit{HC} & \textit{SBCE} & \textit{HBCE} \\
\midrule
DAMCC & \underline{0.5643} & \underline{0.4147} & \underline{1.3867} & \underline{1.2858} \\
RANDOM & 0.5849 & 0.4588 & 1.4433 & 1.4227 \\
DAMNETS & \underline{\textbf{0.2230}} & \textbf{0.0491} & \underline{\textbf{0.3609}} & \textbf{0.1523} \\
AGE & \textbf{0.2483} & \underline{\textbf{0.0465}} & \textbf{0.4432} & \underline{\textbf{0.1442}} \\
\bottomrule
\end{tabular}
\end{table}

\begin{table}[ht]
\centering
\caption{Comparison of Different Losses for $\mathbf{coB}_{0,2}$ across models for England Covid dataset. Lower values are better.}
\label{enc_cc2}
\begin{tabular}{lcccc}

\toprule
 & \textit{SC} & \textit{HC} & \textit{SBCE} & \textit{HBCE} \\
\midrule
DAMCC & \underline{\textbf{0.7914}} & \underline{\textbf{0.6160}} & \underline{\textbf{4.3684}} & \underline{\textbf{4.0027}} \\
RANDOM & 0.8339 & 0.6738 & - & 7.1447 \\
\bottomrule

\end{tabular}
\end{table}

\subsection{Community Decay}

Community Decay Model has no specific source to pinpoint within the literature. We will take the definition and implementation where we found it \cite{clarkson_damnets_2023}: We represent the initial network with a graph $G_0 = (V, E_0)$ adhering to Definition~\ref{graph:main}. Then we define a surjective community attachment function \(f_c :V \rightarrow Q\), where $Q$ is the set of communities. In our case, \(\left| Q\right| = 3\). We initialise the time series by setting \(V_0\) by

\begin{equation}
\Prob((i,j) \in E) =
\begin{cases}
    p_{int} & \text{if } f_c(i) = f_c(j) \\
    p_{ext} & \text{if } f_c(i) \neq f_c(j).
\end{cases}
\end{equation}

We construct the consequent Graphs in the time series as follows: We set a decay community $D \in Q$. We define the set of internal edges for this community $D$ by
\(
D^{int}_t := \{(i,j) \in E_t \mid f_c(i) = f_c(j) = D\}. 
\)

At each timestep $t$, a set proportion $f_{dec}$ of $D^{int}_t$ are replaced with external edges by uniform randomly picking an internal edge $(i, j)$ and removing it from the graph edge set, and then picking one of the two nodes of the edge removed ($i$ or $j$) with probability \(\frac{1}{2}\). Without loss of generality let that node be $i$. We then uniform randomly pick a node $k$ with $f_c(k) \neq D$ s.t. $(i,k) \notin E_t$. We then add $(i,k)$ to $E_t$. We repeat this until we reach final timestep $T$.

For our experiments we set \(T = 40, |Q| = 3, p_{int} = 0.9, p_{ext} = 0.01, f_{dec} = 0.3\), with each community starting off with \(15\) nodes, \(45\) in total. We generate $10$ such time series: $5$ for training, $2$ for validation, and $3$ for testing. For CC evaluation, we lift these graphs to CCs according to the lifting procedure described in \ref{lifting}. Results for the graph and CC time series evaluations are displayed in Tables~\ref{comm_g} and~\ref{comm_cc1},\ref{comm_cc2} respectively.
\begin{table}[ht]
\centering
\caption{Comparison of Models on Various Metrics for Community Decay Dataset. Lower is better.}
\label{comm_g}
\begin{tabular}{lcccc}
\toprule
 & \multicolumn{4}{c}{\textit{Models}} \\
\cmidrule{2-5}
{Metrics} & AGE & DAMCC & DAMNETS & RANDOM \\
\midrule
\textit{spectral} & \textbf{0.003} & 0.022 & \textbf{\underline{0.002}} & \underline{0.007} \\
\textit{degree} & \underline{0.021} & \textbf{0.017} & \textbf{\underline{0.013}} & 0.045 \\
\textit{deg. cent.} & \textbf{\underline{0.000}} & \textbf{\underline{0.000}} & \textbf{\underline{0.000}} & \textbf{\underline{0.000}} \\
\textit{local clust. coeff.} & \textbf{\underline{0.006}} & \underline{0.019} & \textbf{0.008} & 0.034 \\
\textit{close. cent.} & \textbf{\underline{0.002}} & \underline{0.003} & \textbf{\underline{0.002}} & 0.004 \\
\textit{eigen. cent.} & \underline{0.025} & \textbf{\underline{0.017}} & \textbf{0.018} & 0.043 \\
\textit{ave. clust. coeff.} & \textbf{0.158} & 1.836 & \textbf{\underline{0.106}} & \underline{0.475} \\
\textit{transitivity} & \textbf{0.225} & 2.169 & \textbf{\underline{0.128}} & \underline{0.754} \\
\textit{temp. corr.} & \textbf{0.082} & 0.462 & \textbf{\underline{0.077}} & \underline{0.415} \\
\textit{temp. close. diff.} & \textbf{1.969} & \underline{2.917} & \textbf{\underline{0.200}} & 3.761 \\
\textit{temp. clust. coeff. diff.} & \textbf{0.030} & \underline{0.043} & \textbf{\underline{0.023}} & 0.208 \\
\bottomrule
\end{tabular}
\end{table}

\begin{table}[ht]
\centering
\caption{Comparison of Different Losses for $\mathbf{coB}_{0,1}$ across models for Community Decay  dataset. Lower values are better.}
\label{comm_cc1}
\begin{tabular}{lcccc}
\toprule
 & \textit{SC} & \textit{HC} & \textit{SBCE} & \textit{HBCE} \\
\midrule
DAMMC & \underline{0.4223} & \underline{0.2645} & \underline{2.1036} & \underline{2.1751} \\
RANDOM & 0.4481 & 0.3706 & 2.1945 & 2.1824 \\
DAMNETS & \underline{\textbf{0.1366}} & \underline{\textbf{0.0589}} & \underline{\textbf{0.3470}} & \textbf{0.4557} \\
AGE & \textbf{0.1451} & \textbf{0.0614} & \textbf{0.4415} & \underline{\textbf{0.4452}} \\
\bottomrule
\end{tabular}
\end{table}

\begin{table}[ht]
\centering
\caption{Comparison of Different Losses for $\mathbf{coB}_{0,2}$ across models for Community Decay dataset. Lower values are better.}
\label{comm_cc2}
\begin{tabular}{lcccc}

\toprule
 & \textit{SC} & \textit{HC} & \textit{SBCE} & \textit{HBCE} \\
\midrule
DAMCC & \underline{\textbf{0.6601}} & \underline{\textbf{0.5200}} & - & \underline{\textbf{10.7138}} \\
RANDOM & 0.6804 & 0.5565 & - & 18.4546 \\
\bottomrule

\end{tabular}
\end{table}

\subsection{Barab\'asi--Albert}

Barab\'asi--Albert (BA) model was initially introduced in \cite{Albert_2002} to create synthetic datasets with the \emph{scale-free} property. Given a node \(i\), let \(k_i\) denote the degree of node \(i\), which denotes the number of nodes a node is connected to. We say that a graph is scale-free if \(\Prob(k_i = d) \propto \frac{1}{d^\gamma}\), where \(\gamma \in \R\). We follow the description and implementation in \cite{clarkson_damnets_2023}. BA model has two parameters: a set number of nodes \(n\) and a number of edges \(m\) to add at each timestep. The graph is initialised with \(m\) fully connected nodes. At each timestep, we pick an unconnected node and connect to \(m\) many other nodes with the probability of attaching to node \(i\):

\[
\Pi(k_i) = \frac{k_i}{\sum_{j} k_j},
\]

which is called the \emph{preferential attachment property}.
This procedure yields a graph time series of length \(T = n - m\). For our experiments, we have generated $10$ such time-series with $n = 50 \text{ and } m = 4$, yielding $T=46$. We split the time $10$ time series as follows: $5$ for training, $2$ for validation, and $3$ for testing. For CC evaluation, we lift these graphs to CCs according to the lifting procedure described in \ref{lifting}. Results for the graph and CC time series evaluations are displayed in Tables~\ref{ba_g_table} and~\ref{ba_cc_table},\ref{ba_cc2} respectively.

\begin{table}[ht]
\centering
\caption{Comparison of Models on Various Metrics on Barab\'asi--Albert dataset. Lower is better.}
\label{ba_g_table}
\begin{tabular}{lcccc}
\toprule
 & \multicolumn{4}{c}{\textit{Models}} \\
\cmidrule{2-5}
{Metrics} & AGE & DAMCC & DAMNETS & RANDOM \\
\midrule
\textit{spectral} & \textbf{0.027} & \underline{\textbf{0.022}} & \underline{0.043} & 0.152\\
\textit{degree} & \underline{0.039} & \textbf{0.028} & \underline{\textbf{0.021}} & 0.147\\
\textit{deg. cent.} & \underline{\textbf{0.000}} & \underline{\textbf{0.000}} & \underline{\textbf{0.000}} & \underline{\textbf{0.000}}\\
\textit{close. cent.} & \underline{0.045} & \textbf{0.037} & \underline{\textbf{0.023}} & 0.167\\
\textit{eigen. cent.} & \underline{\textbf{0.010}} & \underline{0.051} & \underline{\textbf{0.010}} & 0.143\\
\textit{ave. clust. coeff.} & \underline{\textbf{0.039}} & \underline{0.254} & \textbf{0.070} & 2.208\\
\textit{transitivity} & \textbf{0.205} & \underline{1.237} & \underline{\textbf{0.145}} & 1.395\\
\textit{temp. corr.} & \textbf{0.173} & 0.707 & \underline{\textbf{0.057}} & \underline{0.456}\\
\textit{temp. close. diff.} & \textbf{3.933} & \underline{12.321} & \underline{\textbf{3.089}} & 16.953\\
\textit{temp. clust. coeff. diff.} & \underline{\textbf{0.004}} & \underline{0.030} & \textbf{0.005} & 0.274\\
\bottomrule
\end{tabular}
\end{table}

\begin{table}[ht]
\centering
\caption{Comparison of Different Losses for $\mathbf{coB}_{0,1}$ across models for Barab\'asi--Albert dataset. Lower values are better.}
\label{ba_cc_table}
\begin{tabular}{lcccc}
\toprule
 & \textit{SC} & \textit{HC} & \textit{SBCE} & \textit{HBCE} \\
\midrule
DAMMC & \underline{0.6233} & \underline{0.5569} & \textbf{4.7378} & \underline{4.4555} \\
RANDOM & 0.6466 & 0.5983 & \underline{4.7925} & 4.7867 \\
DAMNETS & \underline{\textbf{0.1497}} & \underline{\textbf{0.0455}} & 302.2819 & \underline{\textbf{0.3639}} \\
AGE & \textbf{0.1590} & \textbf{0.0763} & \underline{\textbf{0.8502}} & \textbf{0.6101} \\
\bottomrule
\end{tabular}
\end{table}

\begin{table}[ht]
\centering
\caption{Comparison of Different Losses for $\mathbf{coB}_{0,2}$ across models for Barab\'asi--Albert dataset. Lower values are better.}
\label{ba_cc2}
\begin{tabular}{lcccc}

\toprule
 & \textit{SC} & \textit{HC} & \textit{SBCE} & \textit{HBCE} \\
\midrule
DAMCC & 0.7680 & 0.6730 & - & \underline{\textbf{9.4650}} \\
RANDOM & \underline{\textbf{0.7596}} & \underline{\textbf{0.6656}} & - & 16.9635 \\
\bottomrule

\end{tabular}
\end{table}

\section{Discussion \& Analysis}

\subsection{Graph Metrics}

Across all the datasets, our experiment results align with those of \cite{souid_temporal_2024} and \cite{clarkson_damnets_2023}. We use the same graph evaluation metrics and implementation used in \cite{souid_temporal_2024}. Note that these graph metrics, although appear with the same name, are slightly different to those in \cite{clarkson_damnets_2023} as they use MMDs to compare graph statistics, we compute DTW distance based on the quantiles to of these statistics. Still, same as in \cite{souid_temporal_2024} and \cite{clarkson_damnets_2023}, DAMNETS mostly outperforms all the models. Unlike in \cite{souid_temporal_2024} and \cite{clarkson_damnets_2023}, we also provide a random model to see whether these models do really capture the underlying distribution of these graphs, or is it just that the benchmark models are inadequate. We see that DAMNETS almost always outperform the random model except for the average cluster coefficient metric for the England Covid dataset. We attribute this to the fact that, unlike Community Decay and Barab\'asi--Albert datasets, this is a more homogeneous and complex real life dataset where cliques are spread across the nodes more randomly. As the random model also samples nodes uniform randomly, cliques form more uniformly. AGE is usually second best across these graph statistics, with DAMCC and random model performing similarly. DAMCC usually slightly outperforms the random model in Community Decay and Barab\'asi--Albert dataset. However, the reason for this is not that DAMCC actually learned the underlying distribution. DAMCC, in its untrained state, is more likely to sample node indexes that are closer to each other due to its tree structure. Once we are at a leaf node corresponding to a graph node with an odd index $i$, we are only 2 decisions away from sampling a graph node with index $i+1$. Upon inspection of Community Decay and Barab\'asi--Albert datasets, we noticed, because of their implementation, we borrowed from \cite{clarkson_damnets_2023}, that it starts with the first connected nodes having the first few indices. This is independent of these datasets' mathematical construction and is simply due to their implementation in our project. Therefore, by chance, even if DAMCC doesn't learn, the incidence matrices it generates are slightly similar to the targets, hence the performance improvement over randomness.

\subsection{CC Metrics}

HBCE is the most reliable metric here, followed by HC as the other two are approximations for these. However, we include them since SC is the loss function used in training, and SBCE was attempted but resulted in frequent division by zero errors during training. AGE and DAMNETS perform similarly across all the datasets, significantly outperforming both DAMCC and the random model. The difference between DAMNETS and AGE is more significant in the BA dataset, where DAMNETS outperforms. Across all the datasets, DAMCC marginally outperforms the random model in 1-cell generation. For reasons outlined in the above section, this does not mean that the model has learned. In 2-cell generation, the difference is significant. However, we claim that this significant difference is not to be attributed to DAMCC learning the underlying distribution again, but that perhaps in its unlearned state, it is less likely to sample 2-cells which contain many cells as opposed to the random model. We say this because the best validation epochs are around 4-25th epochs and there isn't a significant decrease in training loss or validation loss.

\section{Ablation Studies}
Upon conducting the experiments above, given that we couldn't train DAMCC properly, we felt the necessity to conduct ablation studies. We include them here and not in the appendix, because we believe that the results are significant to the overall narrative of this report.

Given the results of the previous section, we hypothesize that the resulting loss landscape from row-wise permutation-invariant losses is either highly complex or flat, with numerous local minima. This complexity has prevented the model from learning effectively, especially for larger networks. To back up this claim, we conducted further experiments. To see if a much smaller dataset would remedy this issue, we created a BA dataset with $n = 6 \text{ and } m = 1$, yielding $T=5$. See a visualisation of this dataset in Figure~\ref{small_ba_img}. We trained four versions of DAMCC on this dataset:

\label{models}
\textbf{Model 1:} DAMCC trained on regular BCE loss for matrices, and removed relaxed Bernoulli from the Algorithm~\ref{algo2}. We keep the Bernoulli in line 10 so this traversal is not entirely deterministic, but we name it ``Deterministic Traversal'' in the plots.

\textbf{Model 2:} DAMCC trained on regular BCE loss for matrices, and with Algorithm~\ref{algo2} as it is. We name this ``Stochastic Traversal'' in the plots.

\textbf{Model 3:} DAMCC trained on regular SC loss, and removed relaxed Bernoulli from the Algorithm~\ref{algo2}. We still keep the Bernoulli in line 10.

\textbf{Model 4:} DAMCC trained on regular SC loss for matrices, and with Algorithm~\ref{algo2} as it is.

Using BCE loss for this data is justified as no $1$-cells disappear and hence keep the indexing the same. We train all the above models only on the $1$-cells as there are no $2$-cells in this small dataset. We have plotted the training and validation losses for all the models above in Figures~\ref{fig:1}, \ref{fig:2}, \ref{fig:3}, \ref{fig:4}. We see that the only model which managed to learn is Model 1 in the above list. This shows us that not only our suspicion of our row-wise permutation invariant functions were true, but also that the stochastic component of our traversal prevented the model from learning. Still, this goes on to show that the simplified version of DAMCC Model 1 can learn from data and sample incidence matrices. The challenge we have not been able to overcome thus far is the indexing issue for larger CCs.

\begin{figure}[ht]
\centering
\includegraphics[width = 1\hsize]{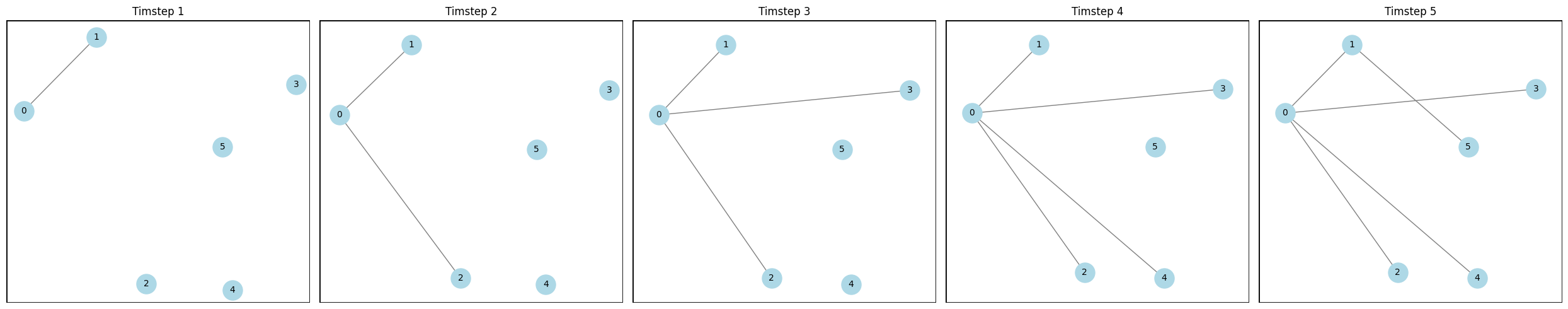}
\caption{A visual representation of the small BA data.}
\label{small_ba_img}
\end{figure}

\begin{figure}[ht]
\centering
\includegraphics[width = 0.8\hsize]{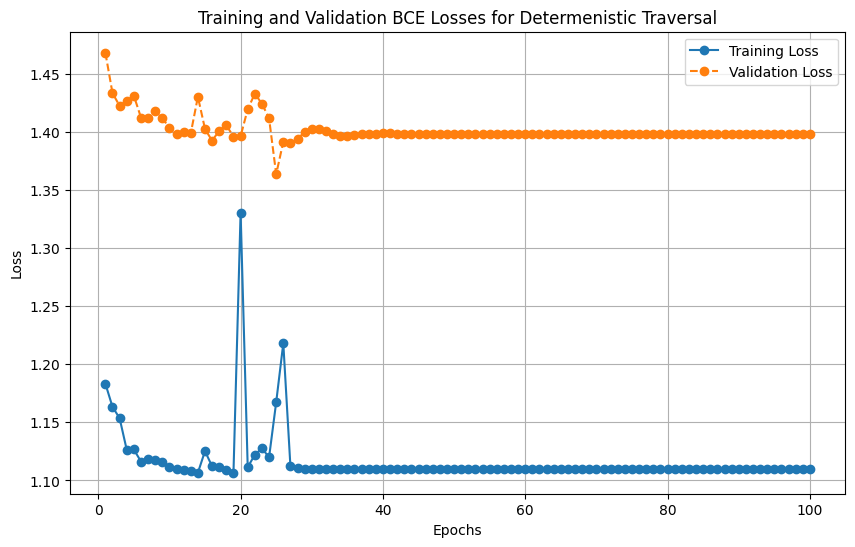}
\caption{Training and validation losses for model 1 in the list in \ref{models}}
\label{fig:1}
\end{figure}

\begin{figure}[ht]
\centering
\includegraphics[width = 0.8\hsize]{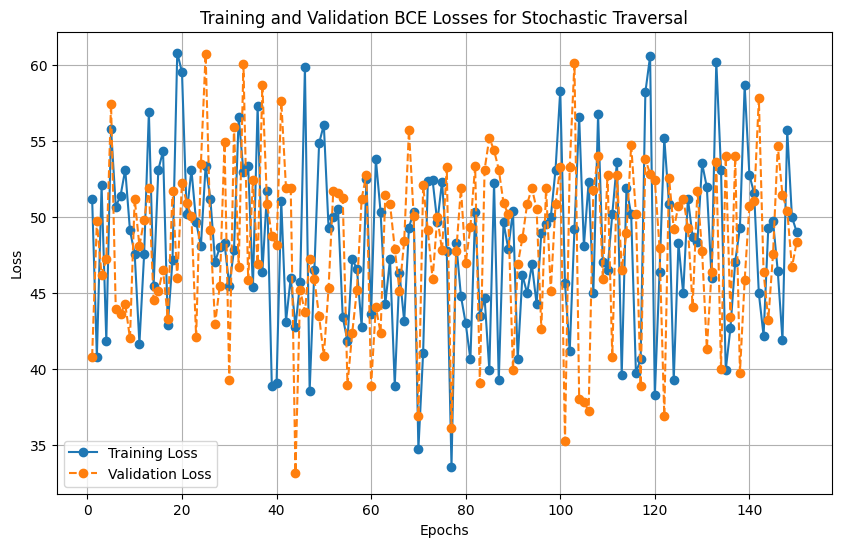}
\caption{Training and validation losses for model 2 in the list in \ref{models}}
\label{fig:2}
\end{figure}

\begin{figure}[ht]
\centering
\includegraphics[width = 0.8\hsize]{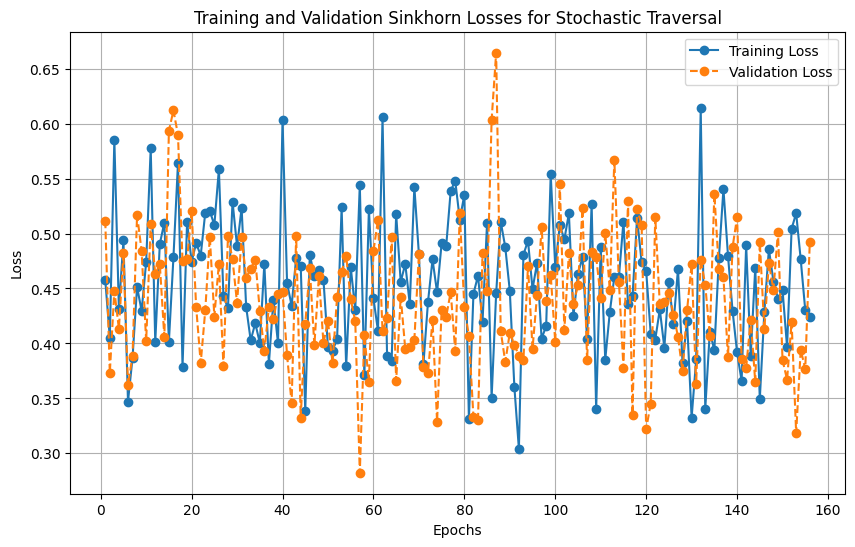}
\caption{Training and validation losses for model 3 in the list in \ref{models}}
\label{fig:3}
\end{figure}

\begin{figure}[ht]
\centering
\includegraphics[width = 0.8\hsize]{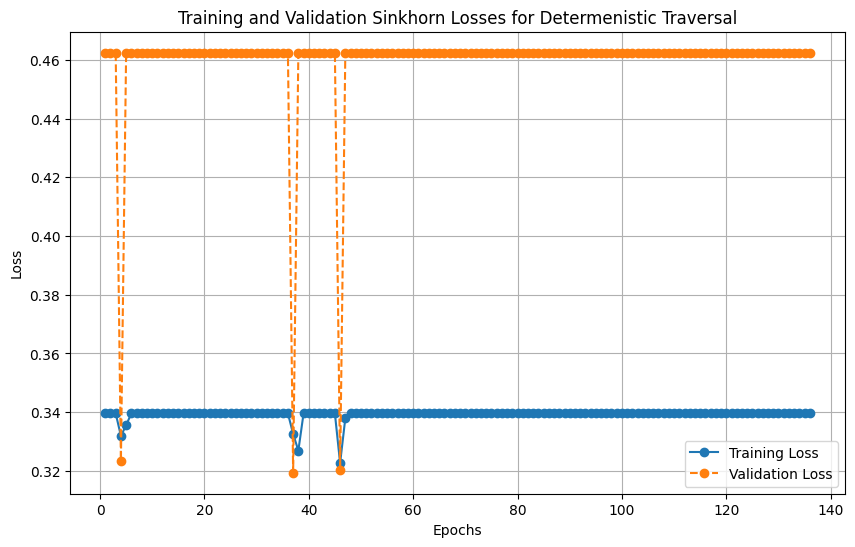}
\caption{Training and validation losses for model 4 in the list in \ref{models}}
\label{fig:4}
\end{figure}

% \begin{table}[ht]
% \centering
% \caption{Comparison of Different Losses across models for BA small dataset. Lower values are better.}
% \label{ablation_table}
% \begin{tabular}{lcccc}
% \toprule
%  & \textit{SC} & \textit{HC} & \textit{BCE} & \textit{HBCE} \\
% \midrule
% bce\_deter & \textbf{0.4057} & \underline{\textbf{0.3646}} & 35.0000 & \textbf{24.3056} \\
% bce\_stochastic & 1.0000 & 1.0000 & \underline{33.3333} & 33.3333 \\
% sinkhorn\_deter & 0.7475 & \underline{0.6792} & \textbf{29.8611} & \underline{\textbf{22.6389}} \\
% sinkhorn\_stochastic & \underline{\textbf{0.3646}} & \underline{\textbf{0.3646}} & \underline{\textbf{25.6944}} & \textbf{24.3056} \\
% random\_ba\_b10 & \underline{0.5145} & \textbf{0.4958} & 44.4444 & \underline{33.0556} \\
% \bottomrule
% \end{tabular}
% \end{table}

%%%%%%%%%%%%%%%%%%%%%%%%%%%%%%%%%%%%
\chapter{Conclusion}

\section{Limitations}

While we successfully developed a model capable of sampling combinatorial complexes (CCs), we were unable to establish an effective framework for training this model. Although the architecture of our model closely resembles that of DAMNETS, training it in a similar manner proved challenging due to fundamental differences in data representation. Specifically, as discussed in Section~\ref{representation}, we cannot index the cells represented as rows in co-incidence matrices the same way as in adjacency matrices. To address this, we developed a row-wise permutation-invariant loss function for training. However, we found evidence that the resulting loss landscape is either highly complex or flat, with numerous local minima. This complexity has prevented the model from learning effectively, especially for larger networks.

Moreover, the assumption that real-world network time series follow a Markovian process may not hold in many cases. For example, in dynamic networks with cyclical structures, predicting the direction of evolution may require accounting for multiple past timesteps rather than relying solely on the previous one.

Additionally, our model is unable to generate 0-cells (nodes) or handle datasets with a variable number of 0-cells. In real-world networks, such as social networks, the appearance of new users and the disappearance of old ones is common, a dynamic our model currently cannot capture.

Furthermore, the datasets we utilized were not inherently in CC form. Instead, we employed a deterministic method to lift graph-based data into CC form. Consequently, all the information in the CCs was already present in the original graphs, meaning no new information was encoded into the higher-rank cells. Future work will require datasets that cannot be accurately represented by graphs without losing critical data.

\section{Achievements}

We have successfully developed a model capable of handling relatively large CC time series by employing a tree-based approach that leverages the sparsity of incidence matrices. To the best of our knowledge, this is the first model in the literature that attempts to generate such time series.

In addition, we have introduced and adapted novel loss and evaluation functions specifically designed for combinatorial complexes, rather than simply adapting techniques from graph-based models. This also represents a new contribution to the field.

Moreover, we evaluated and compared two state-of-the-art models, AGE and DAMNETS, on real-world datasets that they were not originally trained on. This approach, which had not been attempted in previous literature, allowed for a more comprehensive evaluation of model generalization. Finally, we have addressed important issues related to the representation of combinatorial complexes, contributing to future developments in the field.

\section{Future Directions}

Improving the training framework for our model is critical. One potential solution is to explore alternative loss functions that could simplify the optimization landscape, reducing the prevalence of local minima and flat regions. Techniques such as curriculum learning, where the model is trained progressively on increasingly complex datasets, could provide a more structured path toward effective model convergence, especially for larger networks.

Another avenue for improvement lies in relaxing the Markovian assumption for network time series. Future work could extend the model to account for longer temporal dependencies by incorporating encoding a series of CC's rather than a single CC to capture cyclical or long-term trends.

Addressing the inability to generate and handle  $0$-cell (node) counts is also essential for making the model applicable to more dynamic and realistic datasets, such as social networks where nodes may appear and disappear over time. One idea that immediately comes to mind is, as well as sampling rows, we could incorporate an element which samples columns in a similar manner. This way we can effectively sample $0$-cells as well.

At the moment there exist no CC datasets. Future work can focus on acquiring and developing datasets that are inherently structured as CCs, where higher-order interactions cannot be fully represented by graphs. 

Improving the computational efficiency of the model will be essential as larger and more complex networks are tackled. Investigating distributed and parallel computing methods, as well as optimizing the implementation of sparse matrix operations, could allow the model to scale CCs without compromising performance. Much like the Data\_Loader class in Torch Geometric package, we need an efficient way to batch CCs, which our model suffered from the lack of. 

Lastly, once we have methods which can do all of the above, rather than just prediction, we can investigate whether we can strategically perturb these networks to evolve in a desired manner, and test whether these strategies translate to real life.

%% bibliography
\bibliographystyle{plain}  % or another style like 'unsrt', 'abbrv', etc.
\bibliography{references}
% \input{main.bbl}  % Use your .bbl filename
% \bibliographystyle{apa}
% \printbibliography[heading=bibintoc,title={Bibliography}]

\appendix

\chapter{Static Graph Statistics}
\label{appendix_a}

Given an undirected graph \( G \) with \( n \) nodes, let \( A \) be its adjacency matrix, where:
\[
A_{ij} = \begin{cases}
1 & \text{if nodes } i \text{ and } j \text{ are connected,} \\
0 & \text{otherwise.}
\end{cases}
\]

\section{Degree}
\label{degree}

The degree \( k_i \) of node \( i \) is calculated by summing the entries in the \( i \)-th row (or column) of the adjacency matrix:
\[
k_i = \sum_{j=1}^{n} A_{ij}.
\]
Since the graph is undirected and has no self-loops, \( A_{ij} = A_{ji} \) and \( A_{ii} = 0 \).

\section{Degree Centrality}

Degree centrality \( C_D(i) \) of node \( i \) normalizes the degree by the maximum possible degree (\( n - 1 \)):
\[
C_D(i) = \frac{k_i}{n - 1}.
\]
This value ranges from 0 to 1, indicating the centrality of the node in the network.

\section{Local Clustering Coefficient}

The local clustering coefficient \( C(i) \) measures how close the neighbours of node \( i \) are to being a complete graph (clique). It is given by:
\[
C(i) = \frac{2 e_i}{k_i (k_i - 1)},
\]
where \( e_i \) is the number of edges between the neighbours of node \( i \).

Using the adjacency matrix, \( e_i \) can be computed as:
\[
e_i = \frac{1}{2} \sum_{j=1}^{n} \sum_{k=1}^{n} A_{ij} A_{ik} A_{jk}.
\]
Thus, the local clustering coefficient becomes:
\[
C(i) = \frac{ \sum_{j=1}^{n} \sum_{k=1}^{n} A_{ij} A_{ik} A_{jk} }{ k_i (k_i - 1) }.
\]

\section{Closeness Centrality}

Closeness centrality \( C_C(i) \) of node \( i \) is the reciprocal of the average shortest path distance to all other nodes:
\[
C_C(i) = \frac{n - 1}{\sum_{\substack{j=1 \\ j \ne i}}^{n} d(i, j)}.
\]
where \( d(i, j) \) is the length of the shortest path between nodes \( i \) and \( j \).

\section{Eigenvalue Centrality}

Eigenvalue centrality \( C_E(i) \) assigns relative scores to all nodes based on the principle that connections to high-scoring nodes contribute more to the score of the node. It is defined by the eigenvector corresponding to the largest eigenvalue of \( A \):

\[
A x = \lambda_{\text{max}} x,
\]

where:
\begin{itemize}
    \item \( \lambda_{\text{max}} \) is the largest eigenvalue of \( A \).
    \item  \( x \) is the eigenvector associated with \( \lambda_{\text{max}} \). Each component \( x_{i} \) of this vector corresponds to a node in the graph and represents that node's centrality score.
\end{itemize}

The eigenvector \( x \) provides the centrality scores:

\[
C_E(i) = x_i.
\]

\section{Average Clustering Coefficient}

The average clustering coefficient \( \bar{C} \) is the mean of the local clustering coefficients of all nodes:
\[
\bar{C} = \frac{1}{n} \sum_{i=1}^{n} C(i).
\]
This metric gives an overall indication of the clustering in the network.

\section{Transitivity}

Transitivity \( T \) (also known as the global clustering coefficient) measures the overall likelihood that two connected nodes have a common neighbour. It is defined as:

\[
T = \frac{3 \tau}{\Delta}.
\]

Where \( \tau \) and \( \Delta \) are the number of triangles in the graph, mathematically defined below.

\subsection{Number of Triangles \texorpdfstring{$\tau$}{tau}}

The total number of triangles in the graph is:
\[
\tau = \frac{1}{6} \sum_{i=1}^{n} \sum_{j=1}^{n} \sum_{k=1}^{n} A_{ij} A_{jk} A_{ki}.
\]
The factor \( \frac{1}{6} \) accounts for each triangle being counted six times due to permutations of \( i, j, k \).

\subsection{Number of Connected Triplets \texorpdfstring{$\Delta$}{Delta}}

A connected triplet consists of a central node connected to two others. The total number is:
\[
\Delta = \sum_{i=1}^{n} \binom{k_i}{2} = \sum_{i=1}^{n} \frac{k_i (k_i - 1)}{2}.
\]

\section{Spectrum of the Graph Laplacian}

\begin{remark}
    This metric is relatively new, introduced in \cite{liao2020efficientgraphgenerationgraph}.
\end{remark}

\subsection{Degree Matrix \texorpdfstring{$D$}{D}}

The degree matrix \( D \) is a diagonal matrix where each diagonal element represents the degree of node \( i \):
\[
D_{ii} = k_i.
\]
where $k_i$ is as defined in \ref{degree}.

\subsection{Laplacian Matrix \texorpdfstring{$L$}{L}}

The Laplacian matrix \( L \) is defined as:
\[
L = D - A.
\]

\subsection{Eigenvalues and Eigenvectors}

Compute the eigenvalues \( \lambda \) and eigenvectors \( x \) of the Laplacian matrix \( L \):
\[
L x = \lambda x.
\]

The \textbf{spectrum} of the graph Laplacian is the set of eigenvalues \( \lambda_1, \lambda_2, \dots, \lambda_n \) of the Laplacian matrix \( L \).

\chapter{Temporal Graph Statistics}
\label{appendix_b}

\section{Temporal Correlation}

Temporal correlation measures the relationship between the degrees of nodes over time in a dynamic graph. Let \( k_i(t) \) and \( k_j(t) \) be two time series representing the degrees of nodes \( i \) and \( j \) at time \( t \). The temporal correlation \( \text{Corr}(k_i, k_j) \) is given by:

\[
\text{Corr}(k_i, k_j) = \frac{\sum_{t=1}^{T} (k_i(t) - \bar{k_i})(k_j(t) - \bar{k_j})}{\sqrt{\sum_{t=1}^{T} (k_i(t) - \bar{k_i})^2 \sum_{t=1}^{T} (k_j(t) - \bar{k_j})^2}},
\]

where:
\begin{itemize}
    \item \( \bar{k_i} = \frac{1}{T} \sum_{t=1}^{T} k_i(t) \) is the mean degree of node \( i \) over time.
    \item \( \bar{k_j} = \frac{1}{T} \sum_{t=1}^{T} k_j(t) \) is the mean degree of node \( j \) over time.
    \item \( T \) is the total number of time steps.
\end{itemize}

\section{Temporal Closeness}

Temporal closeness for a node in a dynamic graph measures how quickly the node can reach other nodes over time. For a node \( i \) at time \( t \), its temporal closeness centrality \( C_C(i, t) \) is defined as:

\[
C_C(i, t) = \frac{1}{\sum_{j \neq i} d_{ij}(t)}.
\]

where:
\begin{itemize}
    \item \( d_{ij}(t) \) is the shortest path distance from node \( i \) to node \( j \) at time \( t \), considering only the paths available up to time \( t \).
    \item The sum \( \sum_{j \neq i} \) runs over all nodes \( j \) different from \( i \) in the graph.
\end{itemize}

\section{Temporal Clustering Coefficient}

The temporal clustering coefficient measures the tendency of a node's neighbours to form a cluster over time. For a node \( i \) at time \( t \), the temporal clustering coefficient \( C_T(i, t) \) is defined as:

\[
C_T(i, t) = \frac{2 E_{i, t}}{k_i(t) (k_i(t) - 1)}.
\]

where:
\begin{itemize}
    \item \( k_i(t) \) is the degree of node \( i \) at time \( t \), i.e., the number of neighbours of node \( i \) at time \( t \).
    \item \( E_{i, t} \) is the number of edges that exist between the neighbours of node \( i \) up to time \( t \).
\end{itemize}

\chapter{Algorithms}
\label{appendix_c}

\section{Hungarian Algorithm}

Given two matrices \( \mathbf{A} \in \mathbb{R}^{N_A \times D} \) and \( \mathbf{B} \in \mathbb{R}^{N_B \times D} \), we define a pairwise cost (or distance) matrix \( \mathbf{C} \in \mathbb{R}^{N_A \times N_B} \), where each element \( C_{ij} \) represents the cost of assigning the \( i \)-th row of \( \mathbf{A} \) to the \( j \)-th row of \( \mathbf{B} \).

\textbf{Steps of the Hungarian Algorithm:}

\begin{enumerate}
    \item \textbf{Matrix Preparation:} 
    \begin{itemize}
        \item Construct the cost matrix \( \mathbf{C} \). If \( N_A \neq N_B \), pad the smaller matrix with rows of zeros so that \( N_A = N_B \).
    \end{itemize}
    
    \item \textbf{Row and Column Reduction:}
    \begin{itemize}
        \item Subtract the smallest element in each row from all elements of that row.
        \item Subtract the smallest element in each column from all elements of that column.
    \end{itemize}
    
    \item \textbf{Zero-Covering and Augmentation:}
    \begin{itemize}
        \item Cover all zeros in the matrix using a minimum number of horizontal and vertical lines.
        \item If the number of lines equals \( N \) (where \( N \) is the number of rows or columns after padding), proceed to the assignment step. Otherwise, adjust the matrix by subtracting the smallest uncovered element from all uncovered elements and adding it to elements covered twice. Repeat the zero-covering step.
    \end{itemize}
    
    \item \textbf{Optimal Assignment:}
    \begin{itemize}
        \item Assign one-to-one pairs based on the uncovered zeros. Each zero represents a potential assignment.
        \item Ensure the assignment minimizes the total cost \( \sum_{i=1}^{N_A} C_{i\sigma(i)} \).
    \end{itemize}
\end{enumerate}

The permutation \( \sigma \) obtained from this process provides the optimal matching between the rows of \( \mathbf{A} \) and \( \mathbf{B} \). The final loss for the RWPL function is then computed as:

\[
\text{RWPL} = \sum_{i=1}^{N_A} C_{i\sigma(i)}
\]

This ensures that the loss function is invariant to the permutation of rows, making it suitable for comparing co-incidence matrices where cell indices may differ. We do not provide proof that this algorithm is guaranteed to find the best matching, however, the interested reader can find it in the original paper \cite{hungarian}.

\section{Sinkhorn Algorithm}

The Sinkhorn algorithm is used to solve the optimal transport problem approximately. The Gibbs kernel \( K \) is computed from the BCE-based cost matrix \( C \) as:

\[
K = \exp\left(-\frac{C}{\epsilon}\right)
\]

where \( \epsilon \) is the entropy regularization parameter. The scaling vectors \( \mathbf{u} \) and \( \mathbf{v} \) are iteratively updated as follows:

\[
\mathbf{u}^{(t+1)} = \frac{\mu}{K \mathbf{v}^{(t)}}, \quad \mathbf{v}^{(t+1)} = \frac{\nu}{K^T \mathbf{u}^{(t+1)}}
\]

Here, \( \mu \) and \( \nu \) are uniform distributions over the rows of \( \mathbf{A} \) and \( \mathbf{B} \), respectively. The iteration process continues until convergence or the maximum number of iterations is reached.

The optimal transport plan \( \pi \) is then computed as:

\[
\pi = \text{diag}(\mathbf{u}) K \text{diag}(\mathbf{v})
\]

Finally, the Sinkhorn distance, representing the row-wise permutation-invariant loss using BCE, is given by:

\[
\text{Sinkhorn Distance} = \sum_{i=1}^{N_A} \sum_{j=1}^{N_B} \pi_{ij} C_{ij}
\]

This distance measures the minimum cost required to optimally match the rows of \( \mathbf{A} \) to the rows of \( \mathbf{B} \) under the regularized optimal transport plan \( \pi \). The Sinkhorn algorithm is guaranteed to converge under fairly general assumptions. We do not provide a proof for this claim, however, the interested reader can find it in the original paper \cite{Sinkhorn1967ConcerningNM}.

\end{document}